\documentclass[11pt]{article}

\usepackage[letterpaper,margin=1in]{geometry}
\usepackage{lmodern}
\PassOptionsToPackage{numbers, sort&compress}{natbib}

\makeatletter
\renewcommand{\@maketitle}{%
  \newpage\null
  \vskip 1em%
  \begin{center}%
    \let\footnote\thanks%
    {\LARGE\bfseries \@title \par}%
    \vskip 1.5em%
    {\large
      \def\And{%
        \end{tabular}\hfil\linebreak[0]\hfil%
        \begin{tabular}[t]{c}\ignorespaces%
      }%
      \def\AND{%
        \end{tabular}\hfil\linebreak[4]\hfil%
        \begin{tabular}[t]{c}\ignorespaces%
      }%
      \begin{tabular}[t]{c}\@author\end{tabular}\par}%
  \end{center}%
  \par\vskip 1.5em}
\makeatother

\makeatletter
\newcommand{\thanksnomark}[1]{%
  \g@addto@macro\@thanks{%
    \begingroup
      \renewcommand{\@makefnmark}{}%
      \long\def\@makefntext##1{\noindent ##1}%
      \footnotetext{#1}%
    \endgroup
  }%
}
\makeatother
\providecommand{\And}{\quad}
\providecommand{\AND}{\quad}

\usepackage[utf8]{inputenc} 
\usepackage[T1]{fontenc}    
\usepackage{hyperref}       
\usepackage{url}            
\usepackage{booktabs}       
\usepackage{amsfonts}       
\usepackage{nicefrac}       
\usepackage{microtype}      

\usepackage{graphicx}
\usepackage[utf8]{inputenc} 
\usepackage[T1]{fontenc}    
\usepackage{hyperref}
\usepackage{url}
\usepackage{array}
\hypersetup{
  colorlinks   = true, 
  urlcolor     = [rgb]{0,0,1}, 
  linkcolor    = [rgb]{0,0,0.5}, 
  citecolor   = [rgb]{0,0,0.5} 
}

\usepackage{color}
\usepackage{xcolor}
\definecolor{lightgrayish}{gray}{0.8}
  {\endgroup}

\usepackage{url}            
\usepackage{booktabs}       
\usepackage{amsfonts}       
\usepackage{nicefrac}       
\usepackage{microtype}      
\usepackage{xcolor}         
\usepackage{multicol,multirow}
\usepackage{here}
\usepackage{subcaption}
\usepackage{amsmath,amssymb,amsthm,mathtools}
\usepackage{bm}
\usepackage{wrapfig} 
\restylefloat{figure}
\usepackage{multicol} 
\usepackage{cleveref}
\usepackage{natbib}

\makeatother
\theoremstyle{plain}
\newtheorem{theorem}{Theorem}[section]
\newtheorem{proposition}[theorem]{Proposition}
\newtheorem{lemma}[theorem]{Lemma}

\theoremstyle{definition}

\newtheorem{assumption}[theorem]{Assumption}
\theoremstyle{remark}
\newtheorem{remark}[theorem]{Remark}
\usepackage{tikz}
\usepackage{scalerel}

\title{Continual Learning in Modern Hopfield Networks with an Application to Diffusion Models}

\author{%
  Ken Takeda\textsuperscript{1,2}%
  \thanksnomark{\textsuperscript{1}Graduate School of Arts and Science, The University of Tokyo, Tokyo, Japan}%
  \thanksnomark{\textsuperscript{2}Artificial Intelligence Research Center, AIST, Tokyo, Japan}%
  \hspace{2em}%
  Masafumi Oizumi\textsuperscript{1}%
  \hspace{2em}%
  Ryo Karakida\textsuperscript{2, *}%
  \thanksnomark{\textsuperscript{*}Corresponding author, Email: \texttt{karakida.ryo@aist.go.jp}}%
}

\begin{document}

\maketitle

\begin{abstract}
Generative models, including diffusion models, are increasingly used as foundation models and adapted through sequential fine-tuning, making continual learning an essential problem setting. However, continual learning in such generative models remains poorly understood: after a task change, what aspects of the learned distribution are most easily lost, and what replay samples should be prioritized?  We address these questions through the modern Hopfield energy. Recent links between modern Hopfield networks (MHNs) and diffusion models allow analyses in MHNs to be transferred to diffusion models. We introduce intrinsic forgetting as an increase in Hopfield energy after the task change. In tractable  settings in an MHN, we prove that high-energy, outlier-like samples undergo a larger energy increase than cluster-like samples, implying that samples located in sharp, isolated basins are more forgettable. We further analyze memory replay and show that replay is particularly effective for high-energy samples, enabling an energy-based  selection of replay samples. We validate these predictions in experiments on MHNs and two diffusion models under continual-learning settings: Stable Diffusion and a pixel-space DDPM. In these diffusion models, Hopfield energy tracks reconstruction-based forgetting, and replay experiments reveal energy-dependent mitigation of forgetting that is consistent with the MHN analysis.
\end{abstract}

\section{Introduction}
\label{sec:introduction}
Generative models form a foundation of modern AI, and in particular, diffusion models have become a default backbone for image generation \citep{ho2020ddpm,Song2020-kt,Dhariwal2021-un, Rombach2021-yj}.
In recent years,  generative models have increasingly been used not only after training from scratch, but also undergoing continual knowledge transfer, such as continual pretraining and fine-tuning in different domains. 
This brings continual learning (CL) of diffusion models to the foreground as a practical engineering problem and a fast-growing empirical literature~\citep{Zajac2023-cldiffusion,Smith2024-continualdiffusion,GaoLiu2023-dr,Masip2024-gd}, Continual learning for diffusion models remains largely empirical: the mechanism of catastrophic forgetting is still poorly understood, and even basic questions that have been extensively studied in supervised models remain open-- ``\emph{which} information of Task-1 will this model forget after training on Task~2?'' and ``\emph{which} Task-1 samples should we replay to prevent it?''

Our approach to such fundamental questions is through 
Hopfield networks, which have played a central role in advancing the understanding and development of generative models~\citep{hopfield1982neural,krotov2021large}. 
A distinctive feature of Hopfield networks is that they admit an explicitly computable energy function. This enables systematic architecture design with the energy function serving as a guiding principle~\citep{Hoover2023-mp,dehmamy2025nrgpt}, as well as the development of machine learning methods~\citep{Zhang2022-vg,Hofmann2024-hb,scellier2017equilibrium}. In particular, modern Hopfield networks (MHNs) form a family of models that has been actively studied since the rise of deep learning and has been extended to deep learning architectures~\citep{krotov2021large,Ramsauer2021Hopfield}.
Crucially, a growing literature identifies an underlying connection between MHNs and diffusion models: the score function of a Gaussian kernel-density estimator is exactly a Hopfield update~\citep{Ambrogioni2024Associative,Pham2025-kx}, and backward diffusion dynamics can be read as gradient flows on an implicit log-sum-exp energy~\citep{RayaAmbrogioni2023SSB,Ambrogioni2025Thermo}. 

In this work, we reveal that the Hopfield energy serves as 
 a theoretically grounded measure for quantitatively characterizing forgetting and prioritizing which samples to replay. 
Figure~\ref{fig:concept} illustrates  the core idea of quantifying forgetting in generative models admitting the Hopfield energy, whose details we explain in Section \ref{sec:energy-indicator}. 
For such models, changes in generation probability can be tracked by observing the energy. By characterizing the forgetting of knowledge as a decrease in its generation probability, we can reduce the problem of quantifying forgetting to measuring an increase in the corresponding energy.
Regions with densely clustered memory patterns form low energy, less sharp landscapes. Therefore, even if memory vectors shift from Task 1 to Task 2, the  energy increase remains relatively small, so the corresponding states are robust to forgetting. By contrast, a Task-1 memory lying in a region with sparse or isolated memories undergoes a much larger energy increase and forgetting. 
This picture is analogous to the view that associates flat minima in parameter space with better generalization to test samples and out-of-distribution (OOD) samples
\cite{keskar2017large,cha2021swad}.



\begin{figure}[t]
\centering
\includegraphics[width=0.8\linewidth]{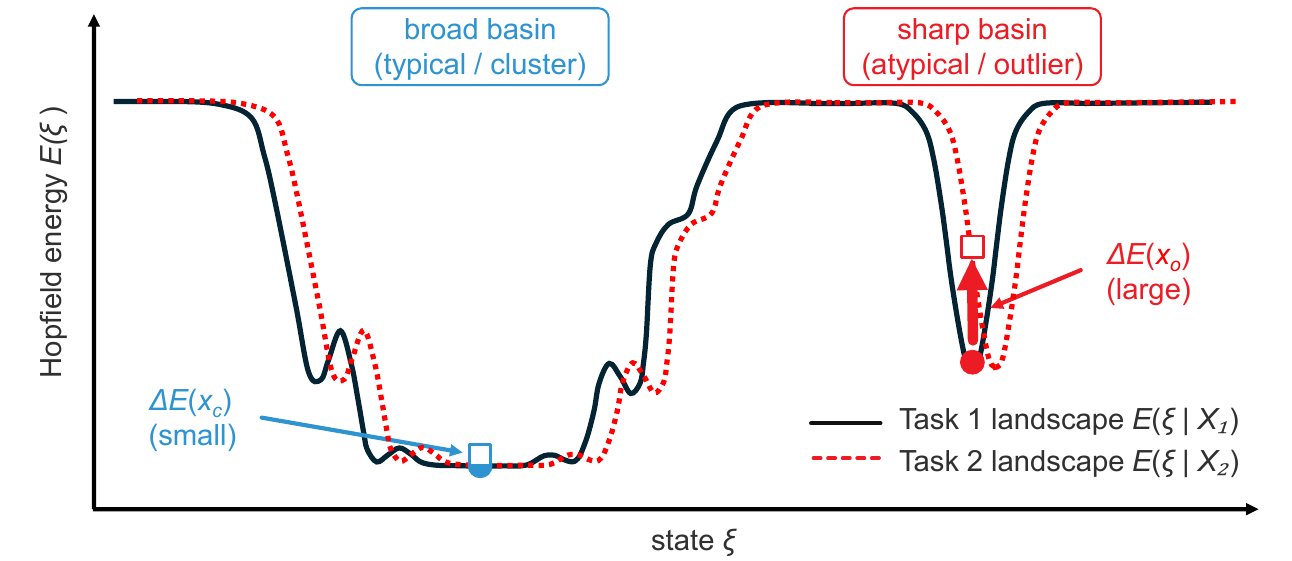}
\caption{\textbf{Conceptual picture.} The  Hopfield energy $E(\xi\mid X)$ of associative memory / diffusion models carves two qualitatively different basins: \emph{broad and low-energy basins} around typical samples, and \emph{sharp and  high-energy basins} around outliers. A task change shifts the landscape (red dashed). The resulting energy rise at the Task-1 memory, $\Delta E(x)=E(x\mid X_2)-E(x\mid X_1)$, referred to as intrinsic forgetting. We prove $\Delta E(x_o)\gg\Delta E(x_c)$: Memories in high-energy basins are more susceptible to forgetting.}
\label{fig:concept}
\end{figure}

Main contributions are summarized as follows: 
\begin{itemize}
    \item \textbf{Hopfield energy view of forgetting supported by theoretical analysis (Theorem ~\ref{thm:forgetting}).}
    We introduce intrinsic forgetting as the increase in Hopfield energy caused by a task transition. 
    This provides an internal measure to quantify how the learned generative landscape changes, without relying solely on external classifiers or distributional metrics such as FID. To formalize the intrinsic forgetting shown in Figure~\ref{fig:concept},  
    we analyze a modern Hopfield network with structured memories. 
    The analysis shows that outlier-like memories, corresponding to the high-energy state and sharp basin, undergo a larger energy increase under a task perturbation than typical cluster memories.

    \item \textbf{Analysis of energy-dependent replay (Theorem ~\ref{thm:replay}).}
    We analyze how Hopfield energy controls the effect of memory replay. 
    The analysis shows that high-energy samples can produce larger replay gains than low-energy samples, revealing an energy dependence in the effect of replay. 
    This suggests that Hopfield energy may provide a useful signal for replay-buffer design. 
    Controlled experiments with an MHN and its Boltzmann-machine counterpart support this energy-dependent replay picture.

    \item \textbf{Empirical examination in diffusion models (Section ~\ref{sec:exp-forgetting}, ~\ref{sec:exp-replay})}.
    We empirically confirm these theoretical predictions using two diffusion models in continual-learning settings: Stable Diffusion~v1.5 fine-tuned on split CIFAR-10 and a pixel-space DDPM trained on the same one. 
    In both settings, Hopfield energy tracks reconstruction-based forgetting, and replay experiments show energy-dependent changes in reconstruction error consistent with the MHN analysis. 
\end{itemize}

Thus, we expect that this work will serve as 
 a theoretical and empirical foundation for exploring forgetting and continual learning in modern generative models.
 

\section{Related Work}
\label{sec:related_works}

\paragraph{Modern Hopfield networks.}

Recent work on MHNs has developed various extensions of the classical Hopfield network~\cite{hopfield1982neural}, including energy functions beyond pairwise correlations~\cite{Krotov2016-sr,Demircigil2017-eh} and architectural generalizations~\citep{krotov2021large}. In particular, \citet{Ramsauer2021Hopfield} showed that a Hopfield network with a log-sum-exp energy function corresponds to softmax attention in modern deep learning architectures. The term MHN is also commonly used to refer specifically to the model of \cite{Ramsauer2021Hopfield}, and we follow this convention.
A distinctive feature of MHNs is that they admit an explicit energy function of states. Recent studies have increasingly exploited this energy as an intrinsic model quantity, for example, to systematically design architectures from its gradient \citep{krotov2021large,Hoover2023-mp}, to detect out-of-distribution (OOD) samples~\cite{Zhang2022-vg,Hofmann2024-hb}, and to derive biologically plausible algorithms~\cite{scellier2017equilibrium}. In this work, we reveal that the energy also works as a useful indicator of forgetting in generative models.

\paragraph{Connection between MHNs and diffusion models.}
\citet{RayaAmbrogioni2023SSB} showed that diffusion trajectories can be viewed as dynamics on an implicit potential landscape, with attractor-like behavior toward the data manifold. 
\citet{Ambrogioni2024Associative} further made the connection to MHNs explicit by showing that, for finite normalized patterns, the low-noise diffusion energy is asymptotically identical to the MHN energy. 
Recent work has more clarified this view~\citep{Ambrogioni2025Thermo} and extended it to memorization/generalization behavior in diffusion models~\citep{Ambrogioni2025Thermo,Pham2025-kx}. 
Our work uses this established MHN--diffusion correspondence in a new setting: continual learning of generative models.

\paragraph{Continual learning in generative models.}
Continual learning in generative models has mainly been studied from a practical perspective, with the aim of mitigating catastrophic forgetting, rather than from the viewpoint of theoretically understanding its mechanism. Early studies considered replay algorithms in Hopfield networks, while much of the recent work has focused on VAEs and GANs, including teacher--student generative modeling, memory-replay GANs, latent space consolidation, and distillation for conditional generation~\citep{Ramapuram2020-lifelong,Wu2018-mergan,Zhai2019-lifelonggan,Deja2021-qj}. More recently, continual learning has also been studied for diffusion models, including experimental demonstrations of forgetting and its mitigation by replay~\citep{Zajac2023-cldiffusion}, a regularization method tailored to diffusion models~\citep{wang2025avoid}, and parameter-efficient adaptation and distillation in text-to-image diffusion models~\citep{Smith2024-continualdiffusion,Masip2024-gd}. In contrast to these application-oriented studies, our work uses MHNs as a theoretical bridge to understand the mechanism of forgetting and to provide a basis for developing principled methods to mitigate forgetting.




\section{Preliminaries}
\label{sec:preliminaries}
\subsection{Modern Hopfield Energy}
\label{sec:mhn-energy}

Let $X=\{x_i\}_{i=1}^{N}$ be a set of stored memories and let $\xi\in\mathbb{R}^d$ denote a state at which the memory landscape is evaluated. 
The continuous modern Hopfield energy is defined as
\begin{equation}
E(\xi\mid X;\beta)
=
-\frac{1}{\beta}\log\sum_{i=1}^{N}\exp(\beta x_i^\top \xi)
+\frac{1}{2}\|\xi\|^2 ,
\label{eq:mhn-energy}
\end{equation}
where $\beta>0$ is an inverse-temperature parameter controlling the sharpness of the energy landscape \cite{Ramsauer2021Hopfield,krotov2021large,Lucibello2024-lo}. 
For $\nabla_\xi E(\xi\mid X;\beta)
=0$, its fixed points satisfy 
\begin{equation}
\xi
=
\sum_{i=1}^{N}p_i(\xi\mid X)x_i, \quad p_i(\xi\mid X)
=
\frac{\exp(\beta x_i^\top \xi)}
{\sum_{j=1}^{N}\exp(\beta x_j^\top \xi)}. 
\label{eq:mhn-fixed-point}
\end{equation}
Throughout this paper, $E(\xi\mid X)$ denotes the Hopfield energy evaluated with respect to the memory set $X$, omitting $\beta$ when it is clear from context. 
The energy is low when the state $\xi$ lies near a cluster of memories, and high when it lies near an isolated memory.
It is the scalar quantity we use below to characterize how the generative landscape changes across tasks.

\subsection{Boltzmann Machine Formulation}
\label{sec:attnbm}

Define the Boltzmann distribution of the MHN by
\begin{equation}
p_\beta(\xi \mid X)
=
\exp(-\beta E(\xi \mid X;\beta))/Z,
\label{eq:attnbm-distribution}
\end{equation}
where \(Z\) is the partition function. One can easily see that, for \(\|x_i\|=1\), \(Z\) is explicitly given by
\(Z = N (2\pi/\beta)^{d/2}\), and this Boltzmann distribution is a Gaussian mixture
$
p_\beta(\xi \mid X)=
\sum_i
\exp\left(
-\frac{\beta}{2}\|\xi - x_i\|^2
\right)/Z.$ In the following section, as one of our synthetic experiments, we perform unsupervised learning of an MHN by maximum-likelihood estimation of this Boltzmann distribution. Specifically, we regard this distribution as a Boltzmann machine (BM) with a trainable weight matrix \(X\), and refer to it as MHN-BM. Since $Z$ is analytically tractable, the model can be trained easily by gradient descent for maximum-likelihood learning. Note that a similar Boltzmann machine counterpart of MHNs was investigated by \cite{OtaKarakida2023-attnbm}, but their treatment of the inverse temperature differs from our focus.

\subsection{Diffusion Models and Implicit Energy}
\label{sec:diffusion-implicit-energy}

Diffusion models are not usually written as explicit energy-based models, but their score can be viewed as the negative gradient of an implicit, time-dependent energy. 
Let \(p_s(x)\) be the noisy data marginal at diffusion time \(s\), and write
\[
E_{\mathrm{DM},s}(x):=-\log p_s(x),
\qquad
\nabla_x\log p_s(x)=-\nabla_x E_{\mathrm{DM},s}(x).
\]
Thus, roughly speaking, reverse sampling moves a noisy sample \(x_s\) by noisy descent on \(E_{\mathrm{DM},s}\); under conservative-drift assumptions, this energy also includes the drift potential~\citep{RayaAmbrogioni2023SSB}. 
Its stable equilibria therefore act as attractors for generation and memory recall~\citep{Ambrogioni2024Associative, Pham2025-kx}. For a finite set of patterns corrupted by additive Gaussian noise,
\[
p_s(x)
\propto
\sum_i
\exp\!\left(
-\|x-y_i\|^2/2\tau_s^2
\right).
\]
Here \(\tau_s>0\) denotes the standard deviation of the Gaussian corruption at diffusion time \(s\). Expanding the square shows that, for normalized memories, \(E_{\mathrm{DM},s}\) reduces up to constants and positive rescaling to
\[
E_{DM, s} (x) \propto
-\frac{1}{\beta_s}
\log\sum_{i=1}^{N}\exp(\beta_s y_i^\top x)
+
\frac{1}{2}\|x\|^2,
\qquad
\beta_s=\tau_s^{-2},
\]
which is the modern Hopfield energy in Eq.~(\ref{eq:mhn-energy}).
We therefore use Hopfield energy as a tractable proxy for the implicit diffusion energy landscape.
We note that this perspective, which involves analysing the reverse (generative) process in terms of attractor dynamics on a Hopfield-like energy landscape, is not unique to this work and has been adopted in prior work~\citep{RayaAmbrogioni2023SSB,Ambrogioni2024Associative,Ambrogioni2025Thermo,Pham2025-kx}; our contribution is to leverage this established viewpoint in the new setting of continual learning of generative models.
A more explicit derivation, including the variance-preserving case, is given in Appendix~\ref{app:diffusion-hopfield-correspondence}.

\section{Energy as an Indicator of Forgetting}
\label{sec:energy-indicator}

We first introduce forgetting with the internal Hopfield energy. Defining ``forgetting'' in an unsupervised generative model is subtle: label-based metrics via external classifiers are heuristic measures~\citep{Zajac2023-cldiffusion,Smith2024-continualdiffusion}, but the generative objective of our interest is essentially unsupervised and class information is not necessarily relevant to the input data.
We need an internal, model-native measure applicable to pure unsupervised settings.

\paragraph{Intrinsic forgetting via Hopfield energy.}
Let $\mathcal{X}_{\mathrm{old}}$ and $\mathcal{X}_{\mathrm{new}}$ denote the stored memory set at task 1 and task 2.
For any state $x$, we define the \emph{intrinsic forgetting} at $x$ as
\begin{equation}
\Delta E(x)
=
E(x\mid\mathcal{X}_{\mathrm{new}})
-
E(x\mid\mathcal{X}_{\mathrm{old}}).
\label{eq:internal-forgetting}
\end{equation}
Large positive $\Delta E(x)$ means that the energy landscape at $x$ has risen, so the model assigns a smaller probability to that state after the task transition. 
Under the assumption that the score function has sufficient expressive capacity to approximate the data distribution of each task nearly perfectly, this quantity captures forgetting induced by the intrinsic and unavoidable distributional differences between tasks.  
It is noteworthy that, for a fixed number of memories, $p_\beta(\xi\mid\mathcal{X}_{\mathrm{new}})/
p_\beta(\xi\mid\mathcal{X}_{\mathrm{old}})
=\exp[-\beta\Delta E(\xi)]$. 
Thus, intrinsic forgetting naturally arises when distributional changes between  tasks are measured in terms of probability density ratios.


\subsection{Theoretical analysis: equal-angular cluster plus outlier under small rotation}
\label{sec:theory1}

We derived an analytical formulation of intrinsic forgetting within a toy setting. 
\paragraph{Setup (Fig.~\ref{fig:setting}).}
Consider memory vectors (corresponding to training samples) $x_1,\dots,x_N,x_o\in\mathbb{S}^{d-1}$ satisfying
\begin{equation}
\begin{split}
x_i^\top x_j=c\ (i\neq j;\ i,j\le N),\qquad x_i^\top x_o=0\ (i\le N),\qquad 0<c<1,
\label{eq:eqang-setup}
\end{split}
\end{equation}
so $\{x_i\}_{i\le N}$ is an \emph{equal-angular cluster} and $x_o$ is the \emph{orthogonal outlier}. A canonical realization in $d\gg 1$ is the latent-Gaussian construction $x_i=\sqrt{c}\,\mu+\sqrt{1-c}\,z_i$ with a fixed unit centroid $\mu$ and i.i.d.\ unit-variance Gaussian residuals $z_i\perp\mu$, for which $\|x_i\|^2=1$ and $x_i^\top x_j=c+O(d^{-1/2})$ concentrate at the desired geometry; the results below depend only on~(\ref{eq:eqang-setup}). Task~1 stores $X=\{x_1,\dots,x_N,x_o\}$; Task~2 
is provided by its rational domain shift: 
\begin{equation}
X_2=VX, \quad \text{with} \ V=\exp(\varepsilon\Omega), \quad \Omega=(W-W^\top)/2.
\label{eq:rotation}
\end{equation}

\begin{wrapfigure}[13]{r}{0.35\linewidth}
\centering
\vspace{-1.5\baselineskip}
\includegraphics[width=0.95\linewidth]{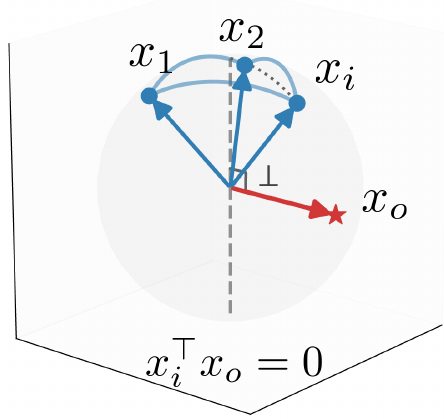}
\caption{\textbf{Memory configuration.} 
}
\label{fig:setting}
\end{wrapfigure}

Note that the rotation identity $E(\xi\mid VX)=E(V^\top\xi\mid X)$ converts a domain shift into a change of argument at the original landscape. This transformation facilitates our analysis. 
In the analysis, we assume a sufficiently small random rotation with
$W_{ij}\stackrel{\mathrm{iid}}{\sim}\mathcal{N}(0,1)$ and $\varepsilon\ll 1$.  In addition, we work at finite $N$, large $\beta$, where individual memories are still fixed points~\citep{Lucibello2024-lo,Achilli2025-cm}.

\paragraph{Fixed points nearby each memory vector.}
 Let $\xi_{cl}^*,\xi_o^*$ denote the cluster-side and outlier-side Task-1 fixed points. 
By the symmetry of~(\ref{eq:eqang-setup}), the softmax weights at each Task-1 fixed point reduce to scalar self-consistency conditions whose large-$\beta$ asymptotics (proved in Appendix~\ref{app:forgetting-proof}) give the closed-form expansions
\begin{equation}
\xi_o^*\sim x_o+e^{-\beta}(s-Nx_o),\quad \xi_{cl}^* \sim x_{cl}+e^{-\beta(1-c)}\!\big(\textstyle\sum_{j\ge 2}x_j-(N-1)x_{cl}\big),
\label{eq:fp-asym}
\end{equation}
with $s:=\sum_{i\le N}x_i$. 
For clarity, here and below we use the informal shorthand ``$\sim$'' to show only the leading-order term; the precise lower-order terms and error bounds are given in Appendix~\ref{app:forgetting-proof} (Eqs.~(\ref{eq:app-xi-out-asym}),~(\ref{eq:app-xi-cl-asym}) for the fixed points and Eqs.~(\ref{eq:app-DE-out})--(\ref{eq:app-DE-cl}) for the energy rises).

The dominant off-memory weights are therefore $e^{-\beta(1-c)}$ for the cluster fixed point and $e^{-\beta}$ for the outlier fixed point, with $e^{-\beta(1-c)}\gg e^{-\beta}$ for $0<c<1$.
Geometrically, the cluster fixed point is supported by $N-1$ near-neighbors at cosine $c$, while the outlier sees only orthogonal competitors.
Appendix~\ref{app:energy-sharpness} shows that these same off-memory weights also control the local energy level and the local sharpness of the basin, as measured by the Hessian trace of the energy landscape.
This justifies the shorthand used below: the cluster fixed point is low-energy and broad, whereas the outlier fixed point is high-energy and sharp.


\paragraph{Intrinsic forgetting at the fixed points.}
For a Task-1 fixed point $\xi_*$, write $\langle\,\cdot\,\rangle_V:=\mathbb{E}_{\Omega}[\,\cdot\,]$ for the expectation over the small-rotation ensemble of~(\ref{eq:rotation}). We absorb this average into the definition and express the (rotation-averaged) intrinsic forgetting at $\xi_*$ as
\begin{equation}
\Delta E_*^{\mathrm{fp}}\;:=\;\big\langle E(\xi_*\mid VX)-E(\xi_*\mid X)\big\rangle_V.
\label{eq:DEfp-def}
\end{equation}
Since $\xi_*$ is stationary, Taylor expansion combined with the rotation identity gives, at leading order,
\begin{equation}
\Delta E_*^{\mathrm{fp}}\;\sim\;\tfrac{\varepsilon^2}{4}\big(\|\xi_*\|^2\,\mathrm{tr}\,H_*-\xi_*^\top H_*\xi_*\big),\qquad H_*=I-\beta\Big(\textstyle\sum_i p_i^* x_ix_i^\top-\xi_*\xi_*^\top\Big).
\label{eq:delta-hessian-body}
\end{equation}
Writing the off-memory weights as $p_k^*=1-\sum_{j\neq k}\eta_j$, $p_j^*=\eta_j\ll 1$ (with $\alpha_j:=x_j^\top x_k$) gives the following:
\begin{equation}
\Delta E_*^{\mathrm{fp}}\;\sim\;\tfrac{\varepsilon^2}{4}\Big[(d-1)-\sum_{j\neq k}\eta_j\,\Lambda(\alpha_j)\Big],\qquad \Lambda(\alpha):=2(d-1)(1-\alpha)+\beta(1-\alpha^2).
\label{eq:delta-eta}
\end{equation}
Specializing \eqref{eq:delta-eta} to the two representative Task-1 fixed points, namely
the outlier ($N$ competitors at $\alpha=0$, $\eta=e^{-\beta}$) and the cluster
($N-1$ competitors at $\alpha=c$, $\eta=e^{-\beta(1-c)}$ plus one competitor at
$\alpha=0$, $\eta=e^{-\beta}$), we obtain the following.

\begin{theorem}[High-energy patterns are forgotten more]
\label{thm:forgetting}
Assume $0<c<1$ and $\beta\gg\log N$. For sufficiently small $\varepsilon$ and large $\beta$,
\begin{equation}
\Delta E_o^{\mathrm{fp}}-\Delta E_{\mathrm{cl}}^{\mathrm{fp}}\;\sim\;\tfrac{\varepsilon^2}{4}(N-1)\big[e^{-\beta(1-c)}\,\Lambda(c)-e^{-\beta}\,\Lambda(0)\big]\;>\;0.
\label{eq:main-ineq}
\end{equation}
\end{theorem}

\noindent\emph{Proof sketch.} $e^{-\beta(1-c)}\gg e^{-\beta}$ at large $\beta$ and $\Lambda(c)>0$ for $0<c<1$, so the first term dominates. The genuine-cluster regime $0<c<1$ is essential: at $c=0$ both fixed points coincide at leading order, while $c<0$ can flip the sign of~(\ref{eq:main-ineq}). The precise lower-order terms and the full derivation are in Appendix~\ref{app:forgetting-proof}.\hfill$\square$

\subsection{Synthetic verification in MHN and Boltzmann-machine learning}
\label{sec:synthetic-forgetting}

We verify Theorem~\ref{thm:forgetting} in synthetic cluster-plus-outlier systems. 
The closed-form MHN uses an equal-angular cluster and an orthogonal outlier with
\(N=12,d=50,c=0.35,\beta=8,\varepsilon=0.05\), averaged over 2000 random rotations. 
The outlier fixed point has a larger rotation-induced energy rise than the cluster fixed points, as predicted by Theorem~\ref{thm:forgetting}; details are in Appendix~\ref{app:toy-forgetting}. 
The same ordering also appears after likelihood training in the MHN-BM with
\(N=12,d=50,c=0.35,\beta=16,\varepsilon_{\mathrm{rot}}=0.02\), shown in Fig.~\ref{fig:fg-attnbm}.

\subsection{Application to Diffusion models}
\label{sec:exp-forgetting}

We next examine whether the prediction of Theorem~\ref{thm:forgetting} extends beyond the tractable MHN setting to continual learning in \emph{diffusion models}. 
The main text focuses on \emph{incremental class learning with split CIFAR-10}, which is a standard benchmark for continual learning. 
We also report results for \emph{domain-incremental learning with rotated MNIST} in Appendix~\ref{app:rotated-mnist}. 
This setting most closely matches the rotation-based analysis in Sections~\ref{sec:theory1} and~\ref{sec:theory2}, and yields the same qualitative conclusions.

We evaluate two practically relevant diffusion-model systems on split CIFAR-10: latent-space \textbf{Stable Diffusion v1.5} fine-tuned on the task split (Fig.~\ref{fig:fg-sd}), and a \textbf{pixel-space DDPM} trained from scratch on the same split (Fig.~\ref{fig:fg-pix}). 
Following prior work~\citep{RayaAmbrogioni2023SSB,Ambrogioni2024Associative,Pham2025-kx}, we define the per-sample Hopfield energy of a diffusion model using Eq.~(\ref{eq:mhn-energy}), with the stored memories $\{x_i\}$ given by the Task-1 training images.

\paragraph{Reconstruction metric for forgetting.}
In the diffusion-model experiments, we additionally use a reconstruction-based measure of forgetting, which evaluates how well a Task-1 sample can still be recovered by the model after Task-2 training.
Specifically, we corrupt a Task-1 sample $x$ to a fixed diffusion timestep $t^\star$, run the post-Task-2 reverse process from $t^\star$ back to $0$, and measure
\begin{equation}
\mathcal{F}_{t^\star}(x)
=
\|x-\hat{x}_{0,t^\star}\|_2^2 .
\label{eq:reconstruction-forgetting}
\end{equation}
A larger $\mathcal{F}_{t^\star}(x)$ indicates that the post-Task-2 reverse dynamics is less able to reconstruct the original Task-1 sample.
We use this heuristic metric to examine whether the energy-based tendency predicted by intrinsic forgetting is also reflected in the behavior of trained diffusion models.
Details of this metric and robustness checks across other timesteps are given in Appendix~\ref{app:diffusion-reconstruction-metric}. We also report a perturbation-based sharpness analysis on CIFAR-10 in Appendix~\ref{app:energy-sharpness-experiment}, showing how local sharpness relates to both Hopfield energy and reconstruction-based forgetting.

\begin{figure}[t]
\centering
\begin{subfigure}[b]{0.33\linewidth}
\centering
\includegraphics[width=\linewidth]{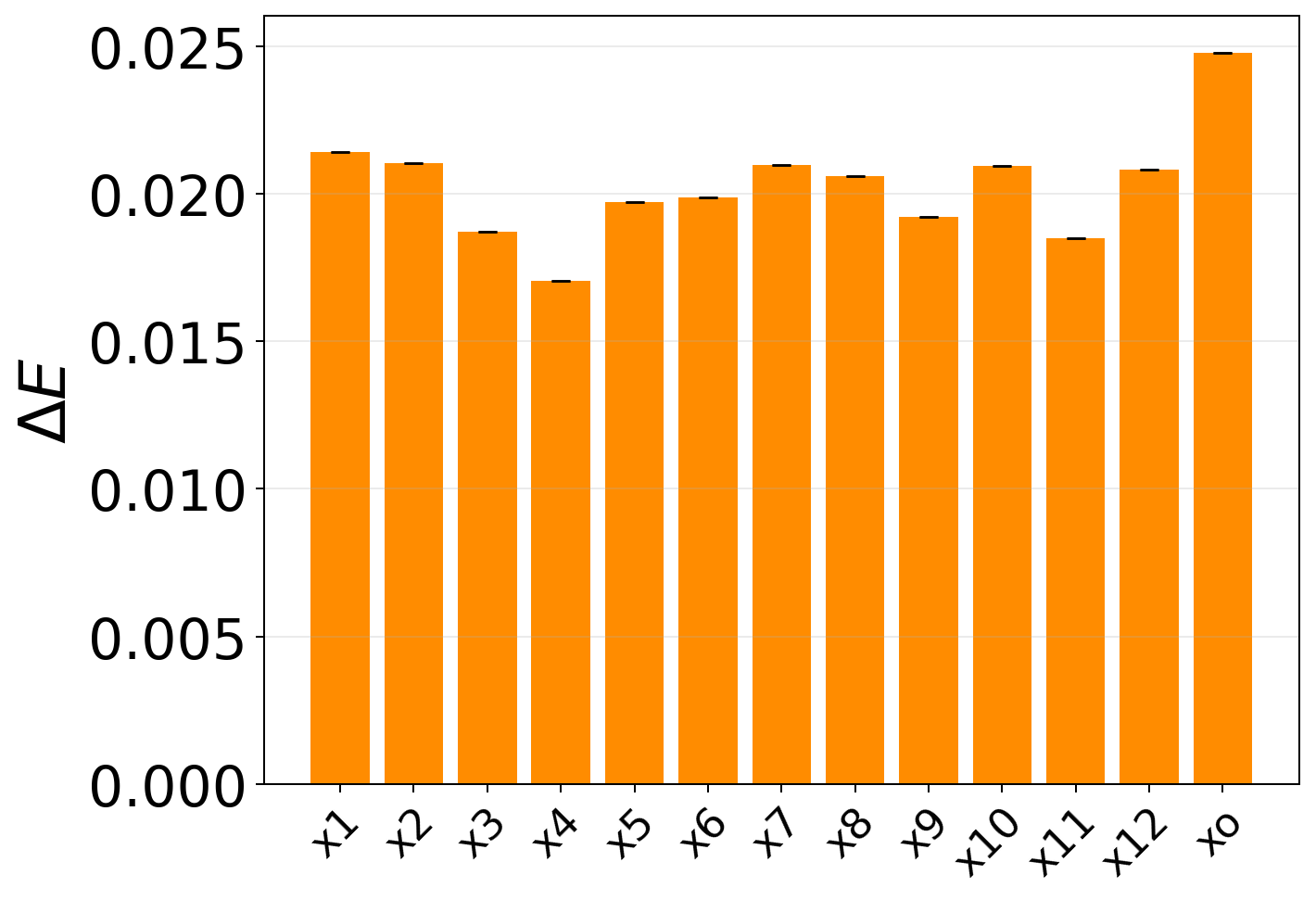}
\caption{MHN-BM.}
\label{fig:fg-attnbm}
\end{subfigure}\hfill
\begin{subfigure}[b]{0.33\linewidth}
\centering
\includegraphics[width=\linewidth]{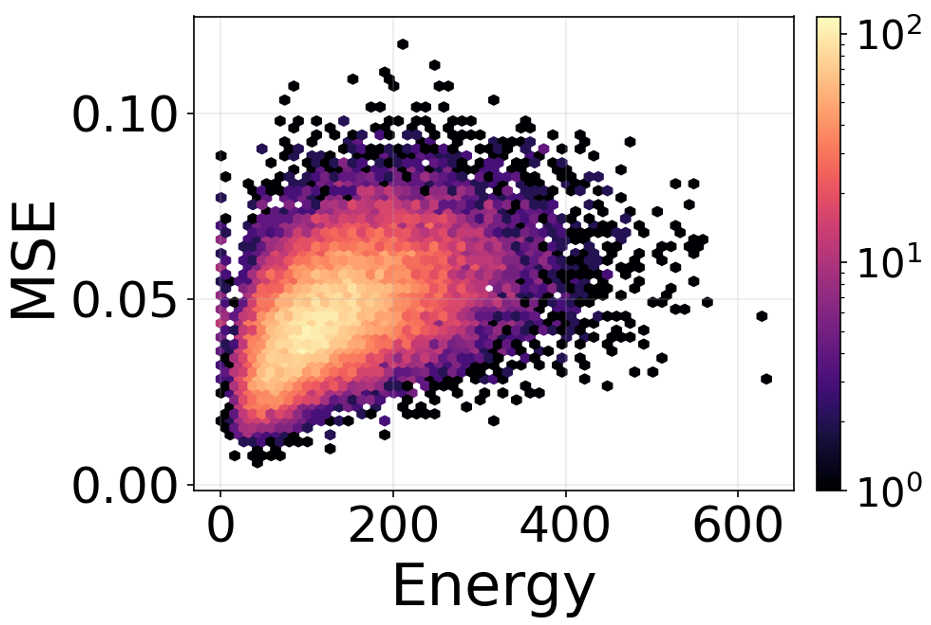}
\caption{Stable Diffusion.}
\label{fig:fg-sd}
\end{subfigure}\hfill
\begin{subfigure}[b]{0.33\linewidth}
\centering
\includegraphics[width=\linewidth]{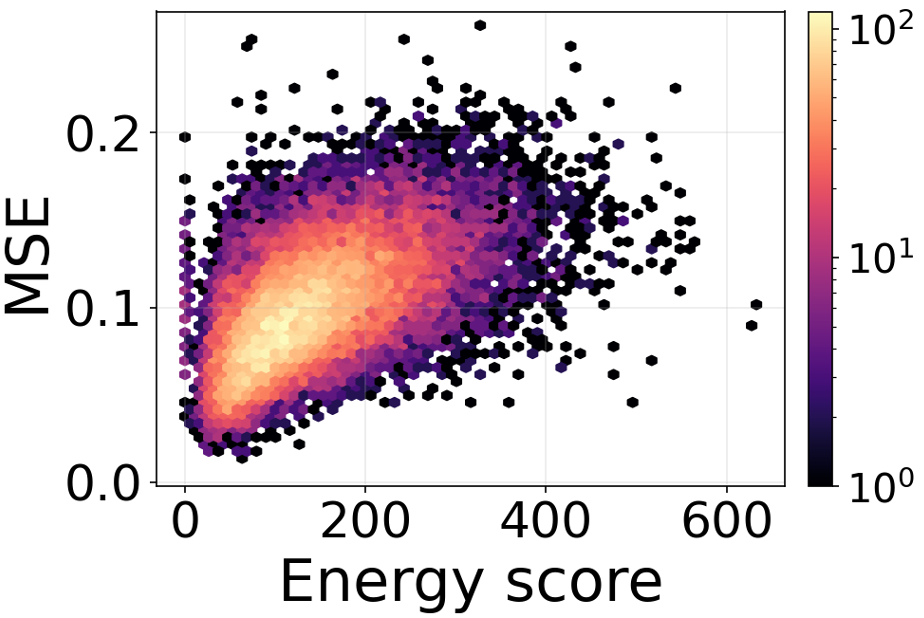}
\caption{Pixel diffusion.}
\label{fig:fg-pix}
\end{subfigure}
\caption{Energy predicts per-sample forgetting across diffusion models and a trainable synthetic model.
(a) In the trainable MHN-BM, the same ordering emerges after likelihood training on the rotated two-task setup. The MHN with synthetic clustering memory is reported in Appendix~\ref{app:toy-forgetting}.
(b) Stable Diffusion v1.5 on split CIFAR-10: per-sample reconstruction MSE $\mathcal{F}_{t^\star}(x)$ after Task~2 training vs.\ internal Hopfield energy measured on Task~1. Higher-energy images suffer larger reconstruction forgetting.
(c) Pixel-space DDPM on the same split reproduces the same monotone relation.
}
\label{fig:forgetting}
\end{figure}

\paragraph{Stable Diffusion on split CIFAR-10 (Fig.~\ref{fig:fg-sd}).}
We fine-tune Stable Diffusion v1.5 (\texttt{runwayml/\ stable-diffusion-v1-5}, VAE frozen, only the UNet trained) on a split CIFAR-10 protocol where each task is composed of different 5 classes. 
After Task~1 we assign every Task-1 image a scalar Hopfield energy $E_\theta(x)$ computed via the log-sum-exp expression~(\ref{eq:mhn-energy}) over all training set of Task-1 images. This is our operationalization of the diffusion model's intrinsic energy at $x$.
After Task~2 training, we measure the forgetting by the per-sample reconstruction MSE at noise level $t=0.8$. This is the main diffusion-model instantiation of Theorem~\ref{thm:forgetting}.

\paragraph{Pixel-space diffusion (Fig.~\ref{fig:fg-pix}).}
To check that the relation is not specific to the latent space learned by Stable Diffusion, we repeat the same protocol with a DDPM trained directly in pixel space on $32\times 32$ split CIFAR-10 from scratch.
The same energy-vs-forgetting ordering is reproduced: per-sample reconstruction MSE after Task~2 is a monotone-increasing function of the Task-1 Hopfield energy. Together, Fig.~\ref{fig:fg-sd} and Fig.~\ref{fig:fg-pix} show that Hopfield energy is a model-agnostic, sample-level forgetting indicator for diffusion CL.

\section{Energy-guided replay prioritization}

Given that high-energy memories are forgotten stronger, the natural question is whether they should be replayed more. We answer this affirmatively with a second theorem, then turn it into a selection rule of replay samples.

\subsection{Theoretical analysis: baseline-subtracted replay gain}
\label{sec:theory2}

We add one Task-1 replay sample $x_r$ to the rotated Task-2 dataset, $X_2^{+r}(V):=VX\cup\{x_r\}$, with the replayed sample counted as an additional term in the log-sum-exp. Define the raw replay gain at the Task-1 fixed point $\xi_k^*$ as
\begin{equation}
\mathcal{R}_{r\to k}(V)\;:=\;E(\xi_k^*\mid X_2(V))-E(\xi_k^*\mid X_2^{+r}(V)).
\label{eq:raw-gain-def}
\end{equation}
Writing $\Delta E_k^{\mathrm{fp}}(V):=E(\xi_k^*\mid X_2(V))-E(\xi_k^*\mid X)$ for the intrinsic forgetting~(\ref{eq:DEfp-def}), the gain decomposes as
\[
\mathcal{R}_{r\to k}(V)
=\underbrace{\Delta E_k^{\mathrm{fp}}(V)}_{\text{forgetting w/o replay}}
\;-\;\underbrace{\bigl[E(\xi_k^*\mid X_2^{+r}(V))-E(\xi_k^*\mid X)\bigr]}_{\text{forgetting w/ replay}},
\]
so $\mathcal{R}_{r\to k}(V)$ measures the reduction of intrinsic forgetting due to replay.
With the softmax weight $p_{r\mid k}=e^{\beta x_r^\top\xi_k^*}/\sum_{\mu\in\mathcal{I}}e^{\beta x_\mu^\top\xi_k^*}$, the log-sum-exp form of $E$ gives $\mathcal{R}_{r\to k}(V)=\tfrac{1}{\beta}\log(1+p_{r\mid k}\,e^{\beta\Delta E_k^{\mathrm{fp}}(V)})$, of which $\mathcal{R}_{r\to k}(I)=\tfrac{1}{\beta}\log(1+p_{r\mid k})$ is the trivial duplicate-memory baseline. Subtracting this baseline and Taylor-expanding in $\Delta E_k^{\mathrm{fp}}(V)=O(\varepsilon^2)$ gives, with $\rho_{r\mid k}:=p_{r\mid k}/(1+p_{r\mid k})$ (full derivation in Appendix~\ref{app:replay-proof}),
\begin{equation}
\Delta_{r\to k}\;:=\;\big\langle\mathcal{R}_{r\to k}(V)-\mathcal{R}_{r\to k}(I)\big\rangle_V\;\sim\;\rho_{r\mid k}\,\Delta E_k^{\mathrm{fp}}.
\label{eq:replay-master}
\end{equation}
We call $\rho_{r\mid k}$ a \emph{replay susceptibility} because it is the linear response of the baseline-subtracted replay gain to the energy rise $\Delta E_k^{\mathrm{fp}}$.
\emph{Replay gain $=$ replay susceptibility $\times$ rotation-induced energy rise.} The first factor asks how much softmax weight the replay point already has at the target fixed point (geometry); the second is exactly the rotation-averaged rise $\Delta E_k^{\mathrm{fp}}$ produced by Theorem~\ref{thm:forgetting} (drift). Let $x_{\mathrm{cl}_1}$ be a cluster memory and let $x_{\mathrm{cl}_2}\neq x_{\mathrm{cl}_1}$ be another member of the same cluster. The replay susceptibilities satisfy
\[
\rho_{\mathrm{cl}_1\mid \mathrm{cl}_1},\,\rho_{o\mid o}\sim \tfrac12,\qquad
\rho_{\mathrm{cl}_1\mid \mathrm{cl}_2}\sim e^{-\beta(1-c)},\qquad
\rho_{\mathrm{cl}_1\mid o},\,\rho_{o\mid \mathrm{cl}_1}\sim e^{-\beta}.
\]
The precise lower-order corrections are in Appendix~\ref{app:replay-proof} (Eqs.~(\ref{eq:app-master}),~(\ref{eq:app-rho-asym})). Combining these with the inequality of $\Delta E_k^{\mathrm{fp}}$ in Theorem \ref{eq:main-ineq}, we obtain the following.

\begin{theorem}[Replay ordering]
\label{thm:replay}
Under the assumptions of Theorem~\ref{thm:forgetting}, the baseline-subtracted replay gain factorises as~(\ref{eq:replay-master}). The five canonical gains obey
\begin{equation}
\Delta_{o \to o}>\Delta_{\mathrm{cl}_1 \to \mathrm{cl}_1}
\gg
\Delta_{\mathrm{cl}_1\to \mathrm{cl}_2}
\gg
\{\Delta_{\mathrm{cl}_1\to o},\Delta_{o\to \mathrm{cl}_1}\},
\qquad
\Delta_{\mathrm{cl}_1\to o}\asymp\Delta_{o\to \mathrm{cl}_1}.\label{eq:replay-order}
\end{equation}
\end{theorem}
\noindent\emph{Intuition.}
Equation~(\ref{eq:replay-master}) says that a replay sample reduces forgetting in proportion to its softmax weight at the target fixed point. 
The first inequality, $\Delta_{o\to o}>\Delta_{\mathrm{cl}_1\to\mathrm{cl}_1}$, compares self-replay: both have susceptibility $\rho\simeq 1/2$, but the outlier target has the larger energy rise by Theorem~\ref{thm:forgetting}, so high-energy self-replay gives the largest per-sample rescue.
The gap $\Delta_{\mathrm{cl}_1\to\mathrm{cl}_1}\gg\Delta_{\mathrm{cl}_1\to\mathrm{cl}_2}$ says that replacing self-replay by another cluster member loses the small overlap factor $e^{-\beta(1-c)}$, although it can still help nearby low-energy states.
Finally, $\Delta_{\mathrm{cl}_1\to\mathrm{cl}_2}\gg\{\Delta_{\mathrm{cl}_1\to o},\Delta_{o\to\mathrm{cl}_1}\}$ and $\Delta_{\mathrm{cl}_1\to o}\asymp\Delta_{o\to\mathrm{cl}_1}$ mean that cross-type replay between the cluster and the outlier is exponentially weaker, because their softmax overlap is only $e^{-\beta}$. 
Thus high-energy replay gives a tall-but-narrow rescue, while low-energy cluster replay gives a shallower-but-broader rescue over neighboring low-energy states.

\begin{wrapfigure}[10]{r}{0.34\linewidth}
\vspace{-3\baselineskip}
\centering
\includegraphics[width=\linewidth]{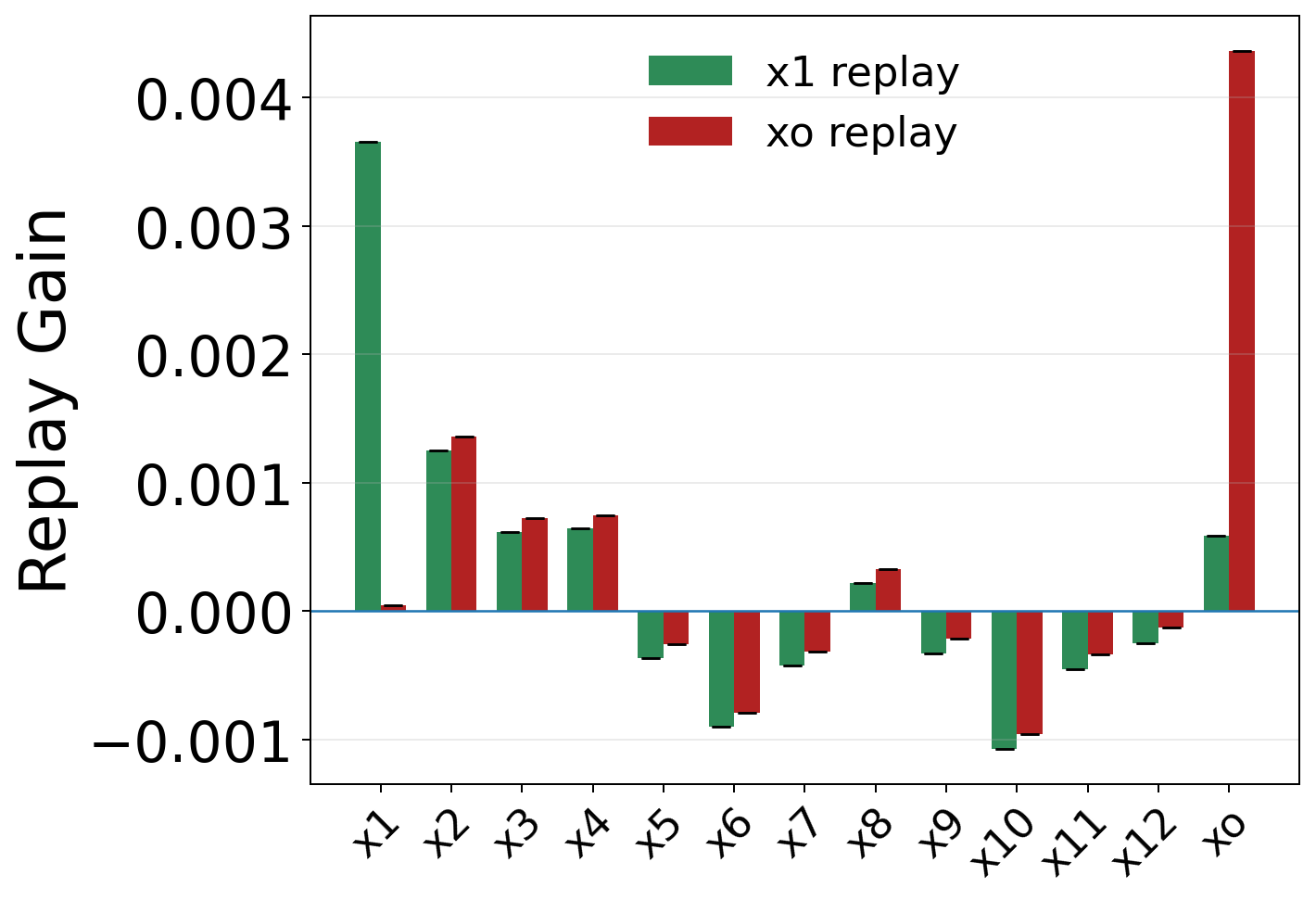}
\caption{MHN-BM corroboration of the replay ordering.}
\label{fig:rp-attnbm}
\end{wrapfigure}
\label{sec:replay}
\subsection{Synthetic verification of replay ordering}
\label{sec:synthetic-replay}

We verify Theorem~\ref{thm:replay} on the same synthetic systems. 
In the closed-form MHN with
\(N=12,d=50,c=0.35,\beta=8,\varepsilon=0.05\), the replay-gain shows the predicted hierarchy:
outlier self-replay is largest, cluster self-replay is next, within-cluster cross replay is smaller, and cross-type replay between the cluster and the outlier is weakest. 
The full result is reported in Appendix~\ref{app:toy-replay}. 
The trained MHN-BM with
\(N=12,d=50,c=0.35,\beta=16,\varepsilon_{\mathrm{rot}}=0.02\)
shows the almost same qualitative ordering in Fig.~\ref{fig:rp-attnbm}.

\subsection{Application to Diffusion models}
\label{sec:exp-replay}

Theorem~\ref{thm:replay} delivers a concrete selection rule of replay samples: among Task-1 samples, replay those that maximize their own energy (thereby inheriting a large $\Delta E_k^{\mathrm{fp}}$) and that self-replay (so that $\rho_{k\mid k}\approx 1/2$). The simplest realization is \emph{energy-guided top-$K$ replay}.

\paragraph{Replay buffer.}
We pick the top-$K$ highest-energy samples of Task 1 as the buffer, $\mathcal{B}=\mathrm{TopK}_{x\in\mathcal{D}_1}E_\theta(x)$, with $K=5000$ in all experiments.
During Task-2 fine-tuning we augment each mini-batch with replay samples drawn from $\mathcal{B}$, so that the effective training set is $\mathcal{D}_2\cup\mathcal{B}$ (25{,}000 new + 5{,}000 replay images per epoch).
We compare against \emph{random} replay (uniform top-$K$) and \emph{no replay} baselines.

We again focus on the two diffusion-model systems of Section~\ref{sec:exp-forgetting} as the main empirical test.


\paragraph{Stable Diffusion (Fig.~\ref{fig:replay-sd}).}
On split CIFAR-10, we compare three replay policies with identical buffer size $K=5000$: 
\emph{energy-guided top-$K$} (the $K$ highest-energy Task-1 images), 
\emph{random}, and \emph{energy-guided bottom-$K$} (the $K$ lowest-energy images). 
After Task-2 fine-tuning, per-sample reconstruction MSE at $t=0.8$ is measured for every Task-1 image and binned by energy.

Eq.~(\ref{eq:replay-master}) predicts a mean--range tradeoff in the replay effect:
$\Delta_{r\to k}\sim \rho_{r\mid k}\,\Delta E_k^{\mathrm{fp}} .$
The energy-rise term $\Delta E_k^{\mathrm{fp}}$ makes the per-sample replay gain larger for high-energy targets, whereas the relevance factor $\rho_{r\mid k}$ determines how broadly a replay sample affects other targets. 
Thus, high-energy replay is expected to produce a large but localized reduction in the high-energy tail, while low-energy replay is expected to produce a smaller reduction spread over a broader set of low-energy samples. 
This is what we observe: the high-energy buffer (Fig.~\ref{fig:rp-sd-hi}) yields a sharply peaked bin-mean reduction at high energy (Fig.~\ref{fig:rp-sd-mean}), but its bin-sum reduction is confined to a narrow energy range (Fig.~\ref{fig:rp-sd-sum}). 
In contrast, the low-energy buffer (Fig.~\ref{fig:rp-sd-lo}) gives a smaller bin-mean reduction at low energy, but the affected range is broader, so the bin-sum reduction extends across a wider energy interval. 
This complementary pattern is consistent with the replay ordering in Eq.~(\ref{eq:replay-order}).

\begin{figure}[t]
\centering
\begin{subfigure}[b]{0.25\linewidth}
\centering
\includegraphics[width=\linewidth]{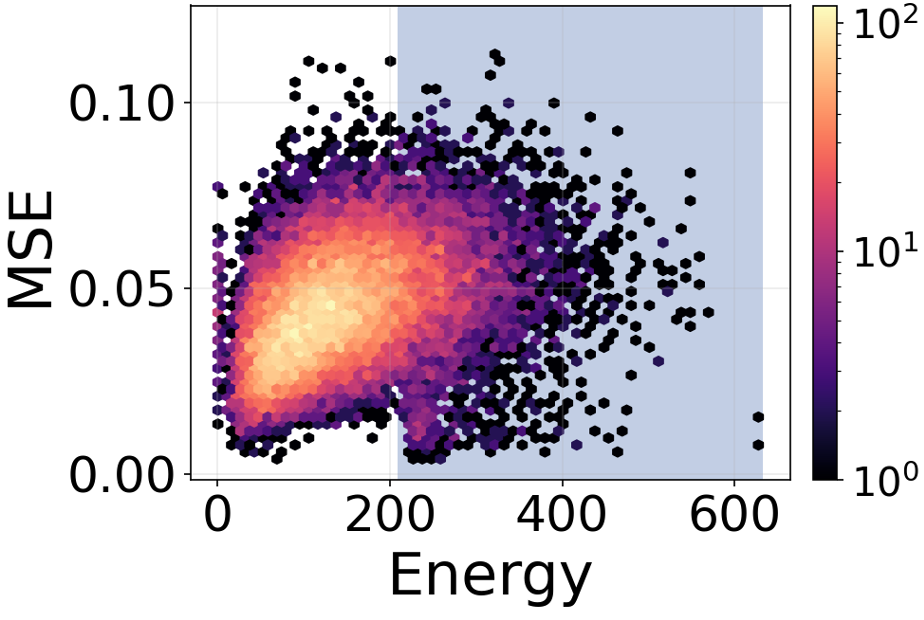}
\caption{High-energy replay.}
\label{fig:rp-sd-hi}
\end{subfigure}\hfill
\begin{subfigure}[b]{0.25\linewidth}
\centering
\includegraphics[width=\linewidth]{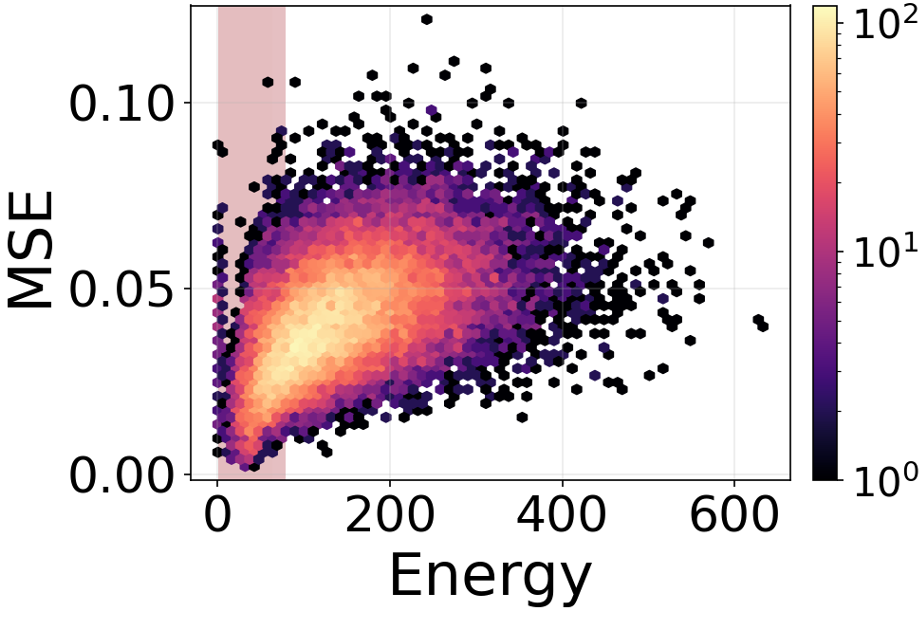}
\caption{Low-energy replay.}
\label{fig:rp-sd-lo}
\end{subfigure}\hfill
\begin{subfigure}[b]{0.25\linewidth}
\centering
\includegraphics[width=\linewidth]{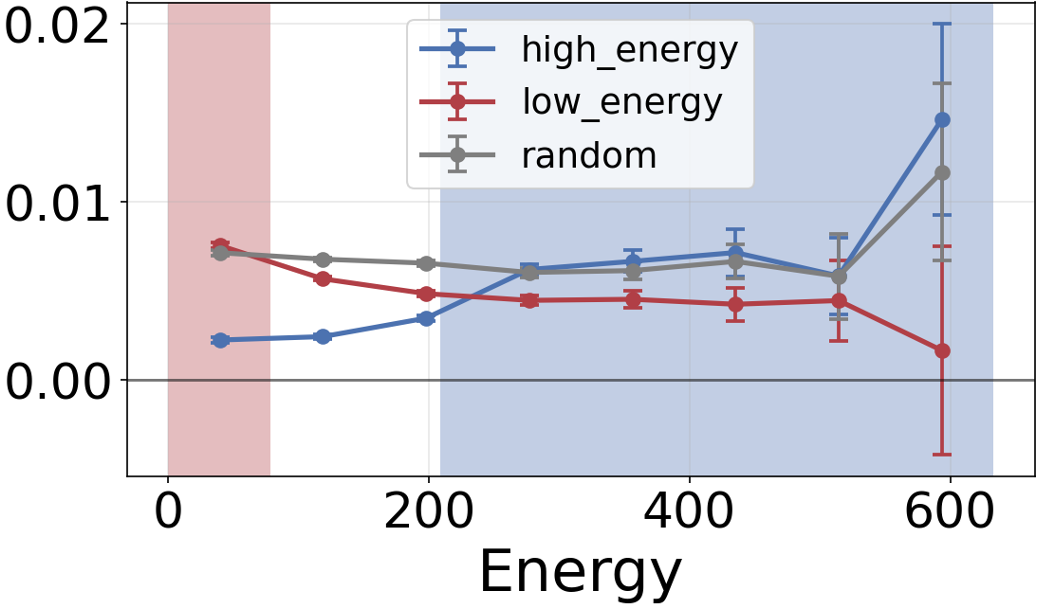}
\caption{Mean Reduction.}
\label{fig:rp-sd-mean}
\end{subfigure}\hfill
\begin{subfigure}[b]{0.25\linewidth}
\centering
\includegraphics[width=\linewidth]{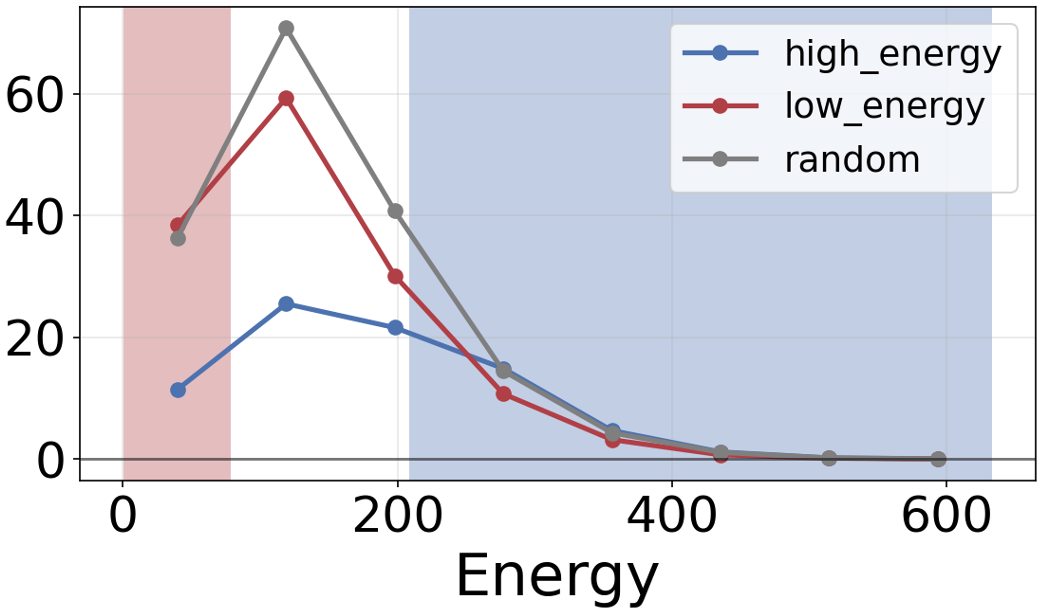}
\caption{Sum Reduction.}
\label{fig:rp-sd-sum}
\end{subfigure}
\caption{Energy-guided replay on Stable Diffusion / split CIFAR-10. (a) Per-sample reconstruction MSE after Task~2 with the top-$K$ \emph{high-energy} replay buffer: MSE drops sharply for high-energy images (strong but narrow rescue). (b) With the bottom-$K$ \emph{low-energy} buffer the per-sample drop is smaller but spread across a wider energy range (weaker but broader rescue). (c,d) Per-bin MSE reduction (baseline $-$ replay) versus energy, averaged (c) and summed (d) over the bin: the high-energy buffer concentrates its rescue at high energy, the low-energy buffer distributes a smaller per-sample rescue over many more samples, exactly as Theorem~\ref{thm:replay} predicts via $\Delta_{r\to k}\approx\rho_{r\mid k}\,\Delta E_k^{\mathrm{fp}}$.}
\label{fig:replay-sd}
\end{figure}
\begin{figure}[t]
\centering
\begin{subfigure}[b]{0.24\linewidth}
\centering
\includegraphics[width=\linewidth]{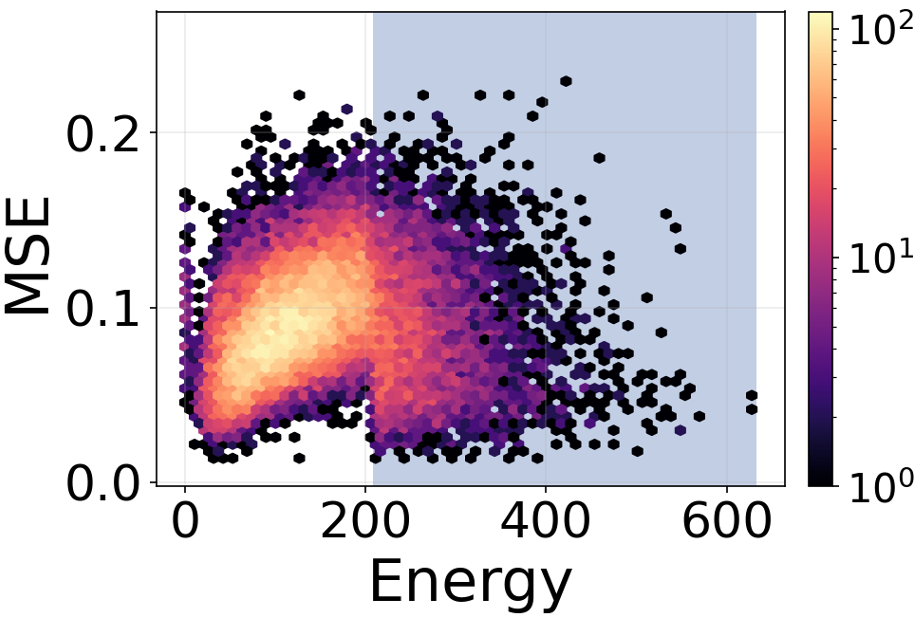}
\caption{High-energy replay.}
\label{fig:rp-pix-hi}
\end{subfigure}\hfill
\begin{subfigure}[b]{0.24\linewidth}
\centering
\includegraphics[width=\linewidth]{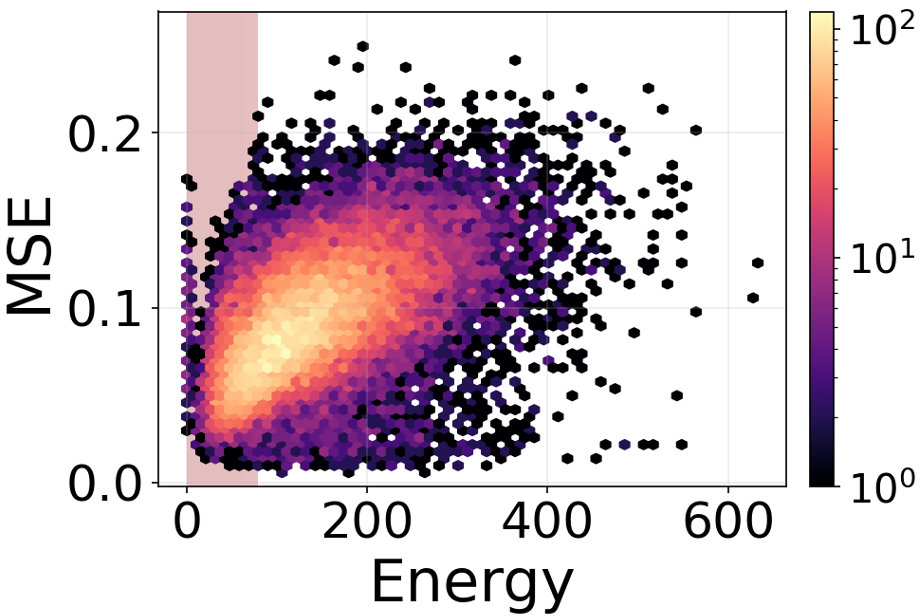}
\caption{Low-energy replay.}
\label{fig:rp-pix-lo}
\end{subfigure}\hfill
\begin{subfigure}[b]{0.24\linewidth}
\centering
\includegraphics[width=\linewidth]{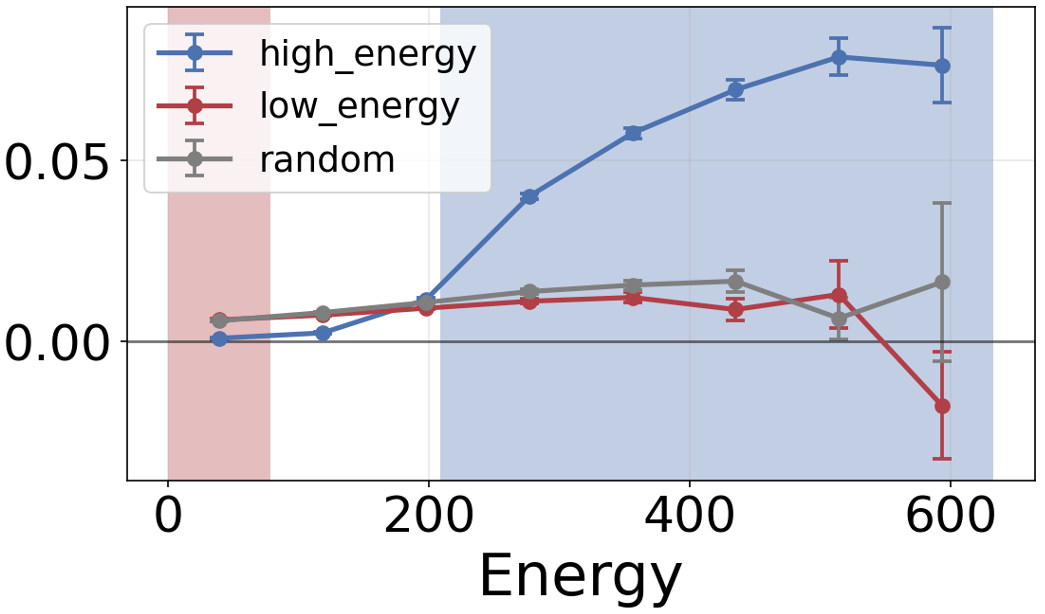}
\caption{Mean Reduction.}
\label{fig:rp-pix-mean}
\end{subfigure}\hfill
\begin{subfigure}[b]{0.24\linewidth}
\centering
\includegraphics[width=\linewidth]{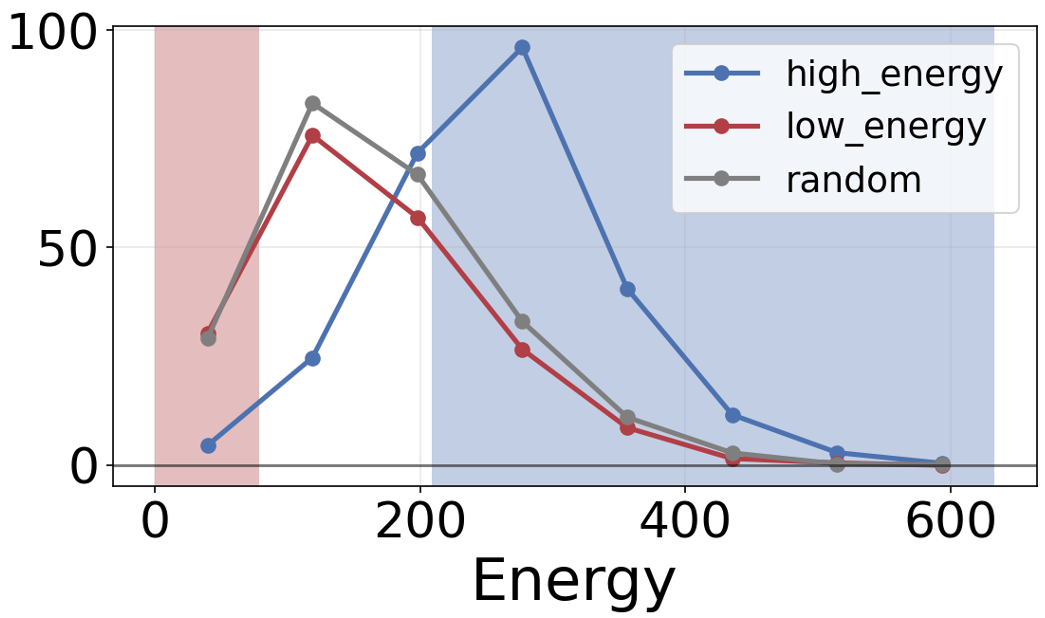}
\caption{Sum Reduction.}
\label{fig:rp-pix-sum}
\end{subfigure}
\caption{The same energy-guided replay analysis on a pixel-space DDPM (trained from scratch on split CIFAR-10) reproduces every qualitative conclusion of Fig.~\ref{fig:replay-sd}: high-energy replay rescues high-energy samples, low-energy replay rescues only low-energy samples, and the total reduction is dominated by the high-energy tail. The finding is not tied to the Stable Diffusion latent space.}
\label{fig:replay-pix}
\end{figure}

\paragraph{Pixel-space diffusion (Fig.~\ref{fig:replay-pix}).}
We repeat the three-policy comparison with the pixel-space DDPM of Section~\ref{sec:exp-forgetting}. Every qualitative conclusion from Stable Diffusion is reproduced: (i) high-energy replay protects high-energy images, (ii) low-energy replay only helps low-energy images, (iii) aggregated reductions place energy-guided top-$K$ above random above low-energy. The replay benefit is again concentrated in the high-energy tail, as predicted. Together, Fig.~\ref{fig:replay-sd} and Fig.~\ref{fig:replay-pix} constitute our main empirical case that Hopfield energy is a practical, diffusion-agnostic replay criterion.

\section{Conclusion and Limitations}
\label{sec:conclusion}

Through theoretical analyses of modern Hopfield networks and empirical confirmation on diffusion models, we find that Hopfield energy serves as an indicator of forgetting and also provides a criterion for selecting replay samples.
These results suggest that forgetting in unsupervised generative models can be evaluated not only through external evaluation metrics, but also through the geometry of the learned state landscape. 
We expect that this work will serve as a theoretical and practical foundation for future studies of knowledge forgetting, replay analysis, and energy-based diagnostics  of generative models.

\paragraph{Limitations.}
(i) Forgetting through the learning of score functions. 
We consider intrinsic forgetting, that is, an unavoidable factor that arises  when a generative model perfectly approximates task-dependent input distributions. In practical diffusion models, however, the score function is represented by a neural network, and additional forgetting may appear even for similar inputs. In supervised learning, theoretical analyses of such forgetting have been developed \cite{lee2022maslow}, and extending the analyses to the score function is an important direction for future work.
(ii) Finite sample sizes.
Although our theory can handle structured data, it is currently restricted to the case of finitely many memories. Quantifying forgetting in regimes where the number of memories grows with the input dimension, and where the memories encode richer data manifolds~\citep{Achilli2025-cm}, remains an open problem.  
(iii) Broader experimental validation. For future practical development, it would also be useful to examine scenarios not covered in this study, such as domains beyond images and online continual learning without explicit task boundaries.
A more elaborated indicator of internal states (e.g., probability-flow ODE marginals) could tighten the correlation of Fig.~\ref{fig:fg-sd},~\ref{fig:fg-pix}.






\bibliographystyle{unsrtnat} 
\bibliography{references/paperpile}

\newpage
\appendix
\renewcommand{\thetable}{S.\arabic{table}}
\renewcommand{\thefigure}{S.\arabic{figure}}
\setcounter{table}{0}
\setcounter{figure}{0}
\renewcommand{\theHtable}{S.\arabic{table}}
\renewcommand{\theHfigure}{S.\arabic{figure}}

\part*{Appendices}

\section{Diffusion--Hopfield Correspondence}
\label{app:diffusion-hopfield-correspondence}

We summarize the correspondence between diffusion energies and modern Hopfield energies. 
Consider a forward noising SDE
\[
dY_s=f(Y_s,s)\,ds+g(s)\,d\widehat W_s,
\]
and let \(p_s(x)\) be the marginal density of \(Y_s\). 
For additive noise and conservative drift, the reverse-time SDE can be written as
\[
dX_t
=
\left[
g^2(T-t)\nabla_x\log p_{T-t}(X_t)
-
f(X_t,T-t)
\right]dt
+
g(T-t)dW_t .
\]
Equivalently,
\[
dX_t=-\nabla_x E_{T-t}(X_t)\,dt+g(T-t)dW_t,
\]
where the implicit energy is
\[
E_s(x)
=
-g^2(s)\log p_s(x)
+
\int_0^x f(z,s)\cdot dz ,
\]
up to an additive constant~\citep{RayaAmbrogioni2023SSB}. 
To examine the attractor structure at a fixed diffusion time \(s\), we freeze the time-dependent coefficients \(g(s)\) and \(f(\cdot,s)\) and study the deterministic drift part of the reverse SDE. That is, we drop only the Brownian increment \(g(s)dW_t\), while keeping the \(g^2(s)\nabla_x\log p_s\) term in the drift. Dropping the Brownian increment is justified in the low-noise (late reverse-time) regime: as \(s\to 0\), the noise scale \(g(s)\to 0\) and the effective inverse temperature of the implicit energy \(\beta_s=\tau_s^{-2}\) (defined below) diverges, which suppresses the stochastic fluctuations and yields exact convergence onto the same patterns that minimize the modern Hopfield energy~\citep{Ambrogioni2024Associative}. Under the conservative-drift assumption above, \(f(\cdot,s)=\nabla\Phi_s(\cdot)\) for a scalar potential \(\Phi_s\), so the line integral in \(E_s\) is path independent. Along the deterministic flow
\[
\frac{d x_\lambda}{d\lambda}
=
-\nabla_x E_s(x_\lambda),
\]
the energy decreases as
\[
\frac{d}{d\lambda}E_s(x_\lambda)
=
-\|\nabla_x E_s(x_\lambda)\|^2
\le 0,
\]
so stable equilibria satisfying \(\nabla_x E_s(x)=0\) define the attractor structure of the reverse dynamics.

Now assume that the data distribution is the empirical distribution over a finite set of patterns,
\[
p_0(y)=\frac{1}{M}\sum_{i=1}^{M}\delta(y-y_i).
\]
In the variance-exploding case,
\[
dY_s=g(s)d\widehat W_s,
\qquad
f=0,
\qquad
\tau_s^2=\int_0^s g^2(r)\,dr .
\]
Thus
\[
Y_s\mid y_i\sim\mathcal N(y_i,\tau_s^2I),
\]
and the noisy marginal is the Gaussian mixture
\[
p_s(x)
\propto
\sum_i
\exp\!\left(
-\frac{\|x-y_i\|^2}{2\tau_s^2}
\right).
\]
Expanding the square gives
\[
p_s(x)
\propto
\exp\!\left(-\frac{\|x\|^2}{2\tau_s^2}\right)
\sum_i
\exp\!\left(
\tau_s^{-2}y_i^\top x
-
\frac{\|y_i\|^2}{2\tau_s^2}
\right).
\]
When the memories are normalized, \(\|y_i\|=r\), the last term is independent of \(i\) and can be absorbed into the normalization. 
Since \(f=0\), we have \(E_s(x)=-g^2(s)\log p_s(x)\). 
Multiplying by the positive factor \(\tau_s^2/g^2(s)\), which does not change the equilibria, yields
\[
-\tau_s^2\log p_s(x)
=
-\tau_s^2
\log\sum_i
\exp(\tau_s^{-2}y_i^\top x)
+
\frac12\|x\|^2
+
\mathrm{const}.
\]
With \(\beta_s=\tau_s^{-2}\), this becomes
\[
E_{\mathrm{MHN}}(x;\beta_s)
=
-\frac{1}{\beta_s}
\log\sum_{i=1}^{M}
\exp(\beta_s y_i^\top x)
+
\frac12\|x\|^2
+
\mathrm{const},
\]
which is the continuous modern Hopfield energy~\citep{Ambrogioni2024Associative}. 
Thus, for discrete normalized patterns, the Gaussian-mixture marginal induced by the forward diffusion defines the same log-sum-exp energy landscape as a modern Hopfield network.

For variance-preserving diffusions, the forward SDE is typically
\[
dY_s=-\frac12 b(s)Y_s\,ds+\sqrt{b(s)}\,d\widehat W_s .
\]
The transition kernel is still Gaussian,
\[
Y_s\mid y_i
\sim
\mathcal N(\theta_s y_i,\tau_s^2 I),
\qquad
\theta_s=
\exp\!\left(
-\frac12\int_0^s b(r)\,dr
\right),
\qquad
\tau_s^2=1-\theta_s^2 .
\]
Therefore \(p_s\) is again a Gaussian mixture, now centered at the shrunk memories \(\theta_s y_i\). 
Substituting this marginal into the implicit energy gives, for normalized memories,
\[
\frac{E_s(x)}{b(s)}
=
-\log\sum_i
\exp\!\left(
\frac{\theta_s}{\tau_s^2}y_i^\top x
\right)
+
\left(
\frac{1}{2\tau_s^2}
-
\frac14
\right)\|x\|^2
+
\mathrm{const}.
\]
Thus the variance-preserving case differs at finite noise by pattern shrinkage and time-dependent coefficients. 
In the low-noise regime, \(\theta_s\to1\) and \(\tau_s\to0\), this reduces to the same log-sum-exp plus quadratic structure, with effective inverse temperature proportional to \(\theta_s/\tau_s^2\). 
More generally, for additive SDEs with conservative drift, the low-noise transition kernel is locally Gaussian around each pattern, so the same Gaussian-mixture argument gives the Hopfield fixed-point structure asymptotically~\citep{Ambrogioni2024Associative}.

\section{Full derivation of Theorem~\ref{thm:forgetting}}
\label{app:forgetting-proof}


We organize the proof as follows. Subsection~\ref{app:setup-rotation} states the equal-angular cluster-plus-outlier model and the rotation identity. Subsection~\ref{app:sym-reduction} reduces the fixed-point equations by symmetry. Subsection~\ref{app:existence-localization} separates the existence of pattern-specific fixed points from the stronger localization condition used to control leakage corrections. Subsection~\ref{app:fp-asym} derives the large-$\beta$ fixed-point asymptotics. Subsection~\ref{app:Hessian} collects the Hessian calculation, the rotation average, and the first-order expansion in the off-memory weights that produce the formula for \(\Delta E_*^{\mathrm{fp}}\). Finally, Subsection~\ref{app:DE-comparison} specializes that formula to the outlier and cluster fixed points and proves Theorem~\ref{thm:forgetting}.

\subsection{Problem setup and rotation identity}
\label{app:setup-rotation}

Let $x_1,\dots,x_N,x_o\in\mathbb{S}^{d-1}$ satisfy~(\ref{eq:eqang-setup}); we assume $0<c<1$ and that the ambient dimension $d$ is large enough for the configuration to exist. A canonical large-$d$ realization is the latent-Gaussian construction $x_i=\sqrt{c}\,\mu+\sqrt{1-c}\,z_i$ with a fixed unit centroid $\mu\in\mathbb{S}^{d-1}$ and i.i.d.\ Gaussian residuals $z_i\perp\mu$ of unit variance, for which $\|x_i\|^2=1$ and $x_i^\top x_j=c+O(d^{-1/2})$ concentrate at the equal-angular geometry; the analysis below depends only on~(\ref{eq:eqang-setup}) and applies to any realization. The MHN energy and update rule are
\begin{equation}
E(\xi\mid X)=-\tfrac{1}{\beta}\log\sum_{i=1}^{N+1}e^{\beta x_i^\top\xi}+\tfrac12\|\xi\|^2,\qquad T(\xi)=\sum_i p_i(\xi)x_i,\qquad p_i(\xi)=\frac{e^{\beta x_i^\top\xi}}{\sum_k e^{\beta x_k^\top\xi}},
\label{eq:app-E-T}
\end{equation}
so that a fixed point $\xi_*=T(\xi_*)$ satisfies $\xi_*=Xp_*$ with $p_*=\mathrm{softmax}(\beta X^\top\xi_*)$. The rotation $X_2=VX$ induces
\begin{equation}
E(\xi\mid VX)=E(V^\top\xi\mid X),
\label{eq:app-rotation-identity}
\end{equation}
reducing Task-2 evaluation at $\xi$ to Task-1 evaluation at $V^\top\xi$.

\subsection{Symmetry reduction of the fixed-point equations}
\label{app:sym-reduction}

Define
\begin{equation}
s:=\sum_{i=1}^N x_i,\qquad A:=1+(N-1)c.
\label{eq:app-s-A}
\end{equation}
A direct calculation using~(\ref{eq:eqang-setup}) gives
\begin{equation}
x_i^\top s=A\quad(i\le N),\qquad x_o^\top s=0,\qquad \|s\|^2=NA.
\label{eq:app-s-properties}
\end{equation}

\paragraph{Outlier-like fixed point.}
The outlier-like fixed point is invariant under permutations of the cluster members, so its softmax weights are
\begin{equation}
p^{\mathrm{out}}=(u,\dots,u,v),\qquad Nu+v=1,
\label{eq:app-p-out}
\end{equation}
and $\xi_{\mathrm{out}}=u\,s+v\,x_o$. The logits read $x_i^\top\xi_{\mathrm{out}}=Au$ for $i\le N$ and $x_o^\top\xi_{\mathrm{out}}=v$. Substituting into the softmax,
\begin{equation}
u=\frac{e^{\beta Au}}{Ne^{\beta Au}+e^{\beta v}},\qquad v=\frac{e^{\beta v}}{Ne^{\beta Au}+e^{\beta v}},
\label{eq:app-out-softmax}
\end{equation}
and taking the log-ratio together with $u=(1-v)/N$ yields the scalar self-consistency
\begin{equation}
\log\frac{Nv}{1-v}=\beta\Big(v-\frac{A}{N}(1-v)\Big).
\label{eq:app-out-scalar}
\end{equation}

\paragraph{Cluster-like fixed point near $x_1$.}
Permuting $x_2,\dots,x_N$ leaves the cluster-like fixed point invariant, so
\begin{equation}
p^{\mathrm{cl}}=(a,b,\dots,b,q),\qquad a+(N-1)b+q=1,
\label{eq:app-p-cl}
\end{equation}
and $\xi_{\mathrm{cl}}=a\,x_1+b\sum_{j\ge 2}x_j+q\,x_o$. The three distinct logits are
\begin{equation}
z_1=a+(N-1)cb,\qquad z_-=ca+\{1+(N-2)c\}b,\qquad z_o=q.
\label{eq:app-z-cl}
\end{equation}
The softmax equations and the identity $z_1-z_-=(1-c)(a-b)$ give the two pattern-selection equations
\begin{equation}
\log\frac{a}{b}=\beta(1-c)(a-b),\qquad \log\frac{a}{q}=\beta(z_1-q).
\label{eq:app-cl-ratios}
\end{equation}
The first is a Curie--Weiss-type fixed-point equation with effective inverse temperature
\begin{equation}
\beta_{\mathrm{eff}}:=\beta(1-c),
\label{eq:app-beta-eff}
\end{equation}
matching the equal-angular capacity scale of~\citet{Achilli2025-cm}.

\subsection{Existence threshold and localization condition}
\label{app:existence-localization}

We separate two requirements: (i) \emph{existence} of pattern-specific cluster fixed points, and (ii) the stronger \emph{localization condition} that keeps the leakage corrections small in the Hessian expansion below.

\subsubsection{Existence threshold}

Drop the outlier coordinate temporarily and consider the symmetric cluster sector $p^{\mathrm{cl}}=(a,b,\dots,b)$ with $a+(N-1)b=1$. Introduce the order parameter
\begin{equation}
m:=a-b,\qquad a=\frac{1+(N-1)m}{N},\qquad b=\frac{1-m}{N}.
\label{eq:app-m-def}
\end{equation}
The pattern-selection equation~(\ref{eq:app-cl-ratios}) becomes
\begin{equation}
\log\frac{1+(N-1)m}{1-m}=\beta_{\mathrm{eff}}\,m,\qquad \beta_{\mathrm{eff}}=\beta(1-c).
\label{eq:app-m-equation}
\end{equation}
The function on the left is convex in $m\in(0,1)$, so~(\ref{eq:app-m-equation}) admits a non-trivial root iff the line of slope $\beta_{\mathrm{eff}}$ touches the convex graph. Tangency at some $m_c\in(0,1)$ defines the critical inverse temperature
\begin{equation}
\beta_c(N):=\min_{m\in(0,1)}\frac{1}{m}\log\frac{1+(N-1)m}{1-m},
\label{eq:app-beta-c-def}
\end{equation}
characterised by
\begin{equation}
\log\frac{1+(N-1)m_c}{1-m_c}=\frac{N\,m_c}{(1-m_c)\{1+(N-1)m_c\}},\qquad \beta_c(N)=\frac{N}{(1-m_c)\{1+(N-1)m_c\}}.
\label{eq:app-mc-tangency}
\end{equation}
Equivalently, with $t_c:=\{1+(N-1)m_c\}/(1-m_c)$,
\begin{equation}
N\,t_c\log t_c=(t_c-1)(t_c+N-1),\qquad \beta_c(N)=\frac{(t_c+N-1)^2}{N\,t_c}.
\label{eq:app-tc-equation}
\end{equation}
The closed form is implicit; the asymptotics is explicit. Elementary bounds on $\log t_c$ yield
\begin{equation}
\log N+1\le \beta_c(N)\le \log N+\log(1+\log N)+1+\frac{\log(1+\log N)}{\log N},
\label{eq:app-betac-bounds}
\end{equation}
so $\beta_c(N)=\Theta(\log N)$, and as $N\to\infty$,
\begin{equation}
\beta_c(N)=\log N+\log\log N+1+o(1).
\label{eq:app-betac-asym}
\end{equation}
Pattern-specific cluster fixed points therefore exist iff
\begin{equation}
\beta(1-c)>\beta_c(N).
\label{eq:app-existence-condition}
\end{equation}

\subsubsection{Leakage and the localization condition}

Existence does not by itself bound the size of corrections in the Hessian expansion below. We use a single auxiliary quantity, the \emph{leakage}: the total softmax mass assigned to memories other than the dominant one. For the outlier and cluster fixed points, define
\begin{equation}
\eta_{\mathrm{out}}:=1-v=Nu,\qquad \eta_{\mathrm{cl}}:=1-a=(N-1)b+q.
\label{eq:app-eta-defs}
\end{equation}
Using $s\perp x_o$ for the outlier and the explicit form of the cluster fixed point,
\begin{equation}
\xi_{\mathrm{out}}-x_o=u\,s-(1-v)\,x_o=\eta_{\mathrm{out}}\Big(\tfrac{s}{N}-x_o\Big),\qquad \xi_{\mathrm{cl}}-x_1=b\sum_{j=2}^N(x_j-x_1)+q(x_o-x_1),
\label{eq:app-distance-form}
\end{equation}
so $\|\xi_{\mathrm{out}}-x_o\|\le C_{\mathrm{out}}\eta_{\mathrm{out}}$ and $\|\xi_{\mathrm{cl}}-x_1\|\le C_{\mathrm{cl}}\eta_{\mathrm{cl}}$ with $c$-dependent constants. The condition must control $\beta\eta$, not only $\eta$, because the leakage enters the Hessian multiplied by $\beta$. Indeed, if $p_k^*=1-\sum_{j\neq k}\eta_j$ and $p_j^*=\eta_j$, then
\begin{equation}
\sum_i p_i^*x_i x_i^\top-\xi_*\xi_*^\top
=
\sum_{j\neq k}\eta_j(x_j-x_k)(x_j-x_k)^\top+O(\eta^2),
\label{eq:app-hessian-leakage-scale}
\end{equation}
and hence $H_*=I-\beta\sum_{j\neq k}\eta_j(x_j-x_k)(x_j-x_k)^\top+O(\beta\eta^2)$. Thus the working localization condition is
\begin{equation}
\beta\,\eta_{\mathrm{out}}\ll 1,\qquad \beta\,\eta_{\mathrm{cl}}\ll 1.
\label{eq:app-localization-cond}
\end{equation}
The next two paragraphs derive the corresponding outlier and cluster forms of this same condition.

\paragraph{Outlier-like fixed point.}
Set $\delta_{\mathrm{out}}:=1-v=\eta_{\mathrm{out}}$ so $u=\delta_{\mathrm{out}}/N$. Substituting into~(\ref{eq:app-out-scalar}),
\begin{equation}
\log\frac{N(1-\delta_{\mathrm{out}})}{\delta_{\mathrm{out}}}=\beta\Big[1-\big(1+\tfrac{A}{N}\big)\delta_{\mathrm{out}}\Big],
\label{eq:app-out-delta-eq}
\end{equation}
or in fixed-point form
\begin{equation}
\frac{\delta_{\mathrm{out}}}{1-\delta_{\mathrm{out}}}=N\,e^{-\beta}\exp\!\big[\beta(1+\tfrac{A}{N})\,\delta_{\mathrm{out}}\big].
\label{eq:app-out-delta-fixed}
\end{equation}
If $\beta N e^{-\beta}\ll 1$, the right-hand side is the contraction $N e^{-\beta}\,(1+O(\beta\delta_{\mathrm{out}}))$ and the leading balance reads $\delta_{\mathrm{out}}=N e^{-\beta}(1+o(1))$. Hence the outlier form of the localization condition is
\begin{equation}
\beta N\,e^{-\beta}\ll 1,
\label{eq:app-out-localization-cond}
\end{equation}
or, for $N\to\infty$, $\beta-\log N-\log\beta\to+\infty$. Under~(\ref{eq:app-out-localization-cond}),
\begin{equation}
v=1-N\,e^{-\beta}(1+o(1)),\qquad u=e^{-\beta}(1+o(1)).
\label{eq:app-out-localization-asym}
\end{equation}

\paragraph{Cluster-like fixed point.}
Equation~(\ref{eq:app-cl-ratios}) gives $b/a=\exp[-\beta(1-c)(a-b)]$. Writing $a=1-\eta_{\mathrm{cl}}$ and $a-b=1-\eta_{\mathrm{cl}}-b$,
\begin{equation}
b=a\,e^{-\beta(1-c)}\exp\!\big[\beta(1-c)(\eta_{\mathrm{cl}}+b)\big].
\label{eq:app-b-fixed}
\end{equation}
Similarly, from $\log(a/q)=\beta(z_1-q)$ with $z_1=a+(N-1)cb$,
\begin{equation}
q=a\,e^{-\beta}\exp\!\big[\beta\{(1-c)(N-1)b+2q\}\big].
\label{eq:app-q-fixed}
\end{equation}
The leading values of these off-memory weights are therefore $(N-1)\,e^{-\beta(1-c)}$ from the other cluster memories and $e^{-\beta}$ from the outlier. The exponential feedbacks in~(\ref{eq:app-b-fixed})--(\ref{eq:app-q-fixed}) are negligible iff the cluster form of the localization condition holds:
\begin{equation}
\beta(1-c)(N-1)\,e^{-\beta(1-c)}\ll 1,\qquad \beta\,e^{-\beta}\ll 1.
\label{eq:app-cl-localization-cond}
\end{equation}
Under~(\ref{eq:app-cl-localization-cond}),
\begin{equation}
\begin{split}
& b=e^{-\beta(1-c)}(1+o(1)),\qquad q=e^{-\beta}(1+o(1)),\\
& a=1-(N-1)\,e^{-\beta(1-c)}-e^{-\beta}+o\big((N-1)\,e^{-\beta(1-c)}+e^{-\beta}\big).
\label{eq:app-cl-localization-asym}
\end{split}
\end{equation}
Equations~(\ref{eq:app-out-localization-cond}) and~(\ref{eq:app-cl-localization-cond}) are simply the outlier and cluster forms of the same localization condition $\beta\eta\ll1$. Since $e^{-\beta(1-c)}\gg e^{-\beta}$ for $0<c<1$, the binding cluster constraint is $\beta(1-c)(N-1)\,e^{-\beta(1-c)}\ll 1$, equivalently $\beta(1-c)-\log N-\log\beta(1-c)\to+\infty$, i.e.
\begin{equation}
\beta(1-c)\ge \log N+\log\log N+\omega(1).
\label{eq:app-beta-localization-final}
\end{equation}

\subsubsection{Working assumption $\beta\gg\log N$}
\label{app:beta-strong}

The capacity threshold $\beta_c(N)=\log N+\log\log N+1+o(1)$ obtained in~(\ref{eq:app-betac-asym}) is the order at which the pattern-specific phase transition occurs. The boundary scale $\beta\sim\log N$ is mathematically interesting and has been the focus of single-task associative-memory analyses, which study how patterns emerge as $\beta$ crosses this threshold. The present paper, however, is about CL --- a more involved setting in which we want to study forgetting and replay at a Task-1 fixed point that is already cleanly localized. We therefore simplify by working in the strictly stronger regime
\begin{equation}
\beta\gg\log N,
\label{eq:app-beta-strong-regime}
\end{equation}
which directly implies the working condition
\begin{equation}
\beta\,N\,e^{-\beta}\ll 1.
\label{eq:app-working-cond}
\end{equation}
For fixed $0<c<1$, $\beta\gg\log N$ also implies $\beta(1-c)\gg\log N$, hence $\beta(1-c)(N-1)\,e^{-\beta(1-c)}\ll 1$, and we obtain the inclusion chain
\begin{equation}
\big\{\beta\gg\log N\big\}\ \subset\ \big\{\beta(1-c)\ge \log N+\log\log N+\omega(1)\big\}\ \subset\ \big\{\beta(1-c)>\beta_c(N)\big\}.
\label{eq:app-beta-inclusion}
\end{equation}
The strict inclusions are useful: the outermost set is the existence threshold; the middle set enforces the localization condition; the innermost set is what we adopt here. Under~(\ref{eq:app-beta-strong-regime}), every Task-1 pattern-specific fixed point is localized around its corresponding memory and the Hessian expansion in the leakage variables is valid.

\begin{assumption}[Working regime]
\label{cond:eqang-working}
We work throughout in the regime $0<c<1$ and $\beta\gg\log N$.
\end{assumption}

\subsection{Large-$\beta$ asymptotics of the Task-1 fixed points}
\label{app:fp-asym}

We collect the closed-form expansions used in the body.

\begin{proposition}[Outlier-like fixed point]
\label{prop:app-fp-out}
Under~(\ref{eq:app-out-localization-cond}),
\begin{equation}
v=1-N\,e^{-\beta}+O(\beta\,e^{-2\beta}),\qquad u=e^{-\beta}+O(\beta\,e^{-2\beta}),
\label{eq:app-vu-asym}
\end{equation}
hence
\begin{equation}
\xi_{\mathrm{out}}=x_o+e^{-\beta}(s-N\,x_o)+O(\beta\,e^{-2\beta}).
\label{eq:app-xi-out-asym}
\end{equation}
\end{proposition}
\begin{proof}
Set $v=1-\delta$ and substitute into~(\ref{eq:app-out-scalar}): $\log\{N(1-\delta)/\delta\}=\beta\{1-\delta-A\delta/N\}$. As $\beta\to\infty$, $\delta\to 0$, so the left-hand side equals $\log(N/\delta)+O(\delta)$ and the leading balance reads $\log(N/\delta)=\beta(1-O(\delta))$, giving $\delta=N\,e^{-\beta}(1+o(1))$. The next-order correction is $O(\beta\delta^2)=O(\beta\,e^{-2\beta})$. Substituting $u=\delta/N$ into $\xi_{\mathrm{out}}=u\,s+v\,x_o$ yields~(\ref{eq:app-xi-out-asym}).
\end{proof}

\begin{proposition}[Cluster-like fixed point]
\label{prop:app-fp-cl}
Under Condition~\ref{cond:eqang-working},
\begin{equation}
b=e^{-\beta(1-c)}(1+o(1)),\qquad q=e^{-\beta}(1+o(1)),\qquad a=1-(N-1)\,e^{-\beta(1-c)}-e^{-\beta}+o\big(e^{-\beta(1-c)}\big),
\label{eq:app-bq-cl}
\end{equation}
hence
\begin{equation}
\xi_{\mathrm{cl}}=x_1+e^{-\beta(1-c)}\Big(\sum_{j=2}^N x_j-(N-1)x_1\Big)+e^{-\beta}\,(x_o-x_1)+o\big(e^{-\beta(1-c)}\big).
\label{eq:app-xi-cl-asym}
\end{equation}
\end{proposition}
\begin{proof}
Since $a\to 1$ and $b,q\to 0$ under the localization condition, the first equation of~(\ref{eq:app-cl-ratios}) gives $\log(a/b)=\beta(1-c)(a-b)=\beta(1-c)(1+o(1))$, so $b=a\,e^{-\beta(1-c)}(1+o(1))=e^{-\beta(1-c)}(1+o(1))$. The second gives $\log(a/q)=\beta(a+(N-1)cb-q)=\beta(1+o(1))$, so $q=e^{-\beta}(1+o(1))$. The sum rule $a+(N-1)b+q=1$ closes the argument. Substituting into $\xi_{\mathrm{cl}}=a\,x_1+b\sum_{j\ge 2}x_j+q\,x_o$ yields~(\ref{eq:app-xi-cl-asym}).
\end{proof}

\begin{remark}
The two leading off-memory weights differ exponentially: the outlier sees $u\sim e^{-\beta}$, while the cluster sees $b\sim e^{-\beta(1-c)}$, and $e^{-\beta(1-c)}/e^{-\beta}=e^{\beta c}\to\infty$. Geometrically, \emph{the cluster fixed point is exponentially broader than the outlier fixed point}; this is the geometric origin of Theorem~\ref{thm:forgetting}.
\end{remark}

\subsection{Hessian expansion for the rotation-averaged energy rise}
\label{app:Hessian}
This subsection collects the Hessian-related calculations used in the proof: the Taylor expansion at a fixed point, the average over small random rotations, and the first-order expansion in the off-memory weights.

\paragraph{Taylor expansion at a Task-1 fixed point.}
Because $\xi_*$ is stationary ($\nabla E(\xi_*\mid X)=0$), Taylor expansion gives
\begin{equation}
\Delta E_*^{\mathrm{fp}}(V)=\tfrac12(V^\top\xi_*-\xi_*)^\top H_*(V^\top\xi_*-\xi_*)+O(\varepsilon^3),\qquad H_*:=\nabla^2 E(\xi_*\mid X).
\label{eq:app-taylor}
\end{equation}
Direct differentiation of~(\ref{eq:app-E-T}) yields
\begin{equation}
H_*=I-\beta\Big(\sum_{i=1}^{N+1}p_i^*x_ix_i^\top-\xi_*\xi_*^\top\Big).
\label{eq:app-Hessian}
\end{equation}
Using $V^\top\xi_*-\xi_*=-\varepsilon\Omega\xi_*+O(\varepsilon^2)$,
\begin{equation}
\Delta E_*^{\mathrm{fp}}(V)=\tfrac{\varepsilon^2}{2}(\Omega\xi_*)^\top H_*(\Omega\xi_*)+O(\varepsilon^3).
\label{eq:app-delta-hessian}
\end{equation}
The covariance term in Eq.~(\ref{eq:app-Hessian}) also explains the $\beta\eta$ scale used in the localization condition. Near a pattern-specific fixed point with $p_k^*=1-\sum_{j\neq k}\eta_j$ and $p_j^*=\eta_j$, it has the expansion
\begin{equation}
\sum_i p_i^*x_i x_i^\top-\xi_*\xi_*^\top
=
\sum_{j\neq k}\eta_j(x_j-x_k)(x_j-x_k)^\top+O(\eta^2).
\label{eq:app-hessian-leakage-expansion}
\end{equation}
Thus leakage contributes to the Hessian at order $\beta\eta$, which is why Eq.~(\ref{eq:app-localization-cond}) controls $\beta\eta$ rather than merely $\eta$.

\paragraph{Rotation average.}
\begin{lemma}
For any $u\in\mathbb{R}^d$, $\langle (\Omega u)(\Omega u)^\top\rangle_V=\tfrac12(\|u\|^2 I-uu^\top)$.
\end{lemma}
\begin{proof}
From $\Omega=\tfrac12(W-W^\top)$ with $W_{ij}\stackrel{\mathrm{iid}}{\sim}\mathcal{N}(0,1)$, $\mathbb{E}[\Omega_{ab}\Omega_{cd}]=\tfrac12(\delta_{ac}\delta_{bd}-\delta_{ad}\delta_{bc})$. Hence $\langle(\Omega u)_i(\Omega u)_j\rangle_V=\tfrac12(\delta_{ij}\|u\|^2-u_iu_j)$.
\end{proof}
Substituting into~(\ref{eq:app-delta-hessian}) and using the main-text convention $\Delta E_*^{\mathrm{fp}}:=\langle\Delta E_*^{\mathrm{fp}}(V)\rangle_V$ from~(\ref{eq:DEfp-def}),
\begin{equation}
\Delta E_*^{\mathrm{fp}}=\tfrac{\varepsilon^2}{4}(\|\xi_*\|^2\,\mathrm{tr}\,H_*-\xi_*^\top H_*\xi_*)+O(\varepsilon^3).
\label{eq:app-delta-avg}
\end{equation}

\paragraph{First-order expansion near a pattern.}\label{app:leakage-exp}
Suppose $\xi_*$ is concentrated at a single stored pattern $x_k$, with weights
\begin{equation}
p_k^*=1-\sum_{j\neq k}\eta_j,\qquad p_j^*=\eta_j\ll 1,\qquad \alpha_j:=x_j^\top x_k.
\label{eq:app-eta-leak}
\end{equation}
The fixed-point identity $\xi_*=Xp_*$ then reads
\begin{equation}
\xi_*=x_k+\sum_{j\neq k}\eta_j(x_j-x_k).
\label{eq:app-xi-eta}
\end{equation}
We compute the three quantities entering~(\ref{eq:app-delta-avg}) step by step, retaining all $O(\eta)$ terms and dropping $O(\eta^2)$.

\smallskip\noindent\textbf{(i) $\|\xi_*\|^2$.} Using $\|x_k\|^2=1$, $x_k^\top(x_j-x_k)=\alpha_j-1$, and $(x_i-x_k)^\top(x_j-x_k)=O(1)$,
\begin{equation}
\|\xi_*\|^2=1+2\sum_{j\neq k}\eta_j(\alpha_j-1)+O(\eta^2)=1-2\sum_{j\neq k}\eta_j(1-\alpha_j)+O(\eta^2).
\label{eq:app-leak-norm}
\end{equation}

\smallskip\noindent\textbf{(ii) $\mathrm{tr}\,H_*$.} From~(\ref{eq:app-Hessian}), using $\|x_i\|^2=1$ and $\sum_i p_i^*=1$,
\begin{equation}
\mathrm{tr}\,H_*=d-\beta\Big(\sum_{i=1}^{N+1}p_i^*\,\|x_i\|^2-\|\xi_*\|^2\Big)=d-\beta\big(1-\|\xi_*\|^2\big),
\label{eq:app-leak-tr-step1}
\end{equation}
and substituting~(\ref{eq:app-leak-norm}),
\begin{equation}
\mathrm{tr}\,H_*=d-2\beta\sum_{j\neq k}\eta_j(1-\alpha_j)+O(\eta^2).
\label{eq:app-leak-tr}
\end{equation}

\smallskip\noindent\textbf{(iii) $\xi_*^\top H_*\xi_*$.} From~(\ref{eq:app-Hessian}),
\begin{equation}
\xi_*^\top H_*\xi_*=\|\xi_*\|^2-\beta\Big(\sum_{i}p_i^*\,(x_i^\top\xi_*)^2-\|\xi_*\|^4\Big).
\label{eq:app-leak-xHx-step1}
\end{equation}
Using~(\ref{eq:app-xi-eta}), $x_k^\top\xi_*=1-\sum_{j\neq k}\eta_j(1-\alpha_j)$, and $x_j^\top\xi_*=\alpha_j+O(\eta)$ for $j\neq k$. Therefore
\begin{equation}
\begin{split}
\sum_i p_i^*(x_i^\top\xi_*)^2
&=\Big(1-\sum_{j\neq k}\eta_j\Big)\Big(1-\sum_{j\neq k}\eta_j(1-\alpha_j)\Big)^{\!2}+\sum_{j\neq k}\eta_j\,\alpha_j^2+O(\eta^2)\\
&=1-2\sum_{j\neq k}\eta_j(1-\alpha_j)-\sum_{j\neq k}\eta_j(1-\alpha_j^2)+O(\eta^2),
\label{eq:app-leak-sum-step}
\end{split}
\end{equation}
where in the last step we used $-\sum_j\eta_j+\sum_j\eta_j\alpha_j^2=-\sum_j\eta_j(1-\alpha_j^2)$ and absorbed cross-products of $\eta$'s into $O(\eta^2)$. Combining~(\ref{eq:app-leak-sum-step}) with $\|\xi_*\|^4=1-4\sum_{j\neq k}\eta_j(1-\alpha_j)+O(\eta^2)$ from~(\ref{eq:app-leak-norm}) gives
\begin{equation}
\sum_i p_i^*(x_i^\top\xi_*)^2-\|\xi_*\|^4=\sum_{j\neq k}\eta_j\big\{2(1-\alpha_j)-(1-\alpha_j^2)\big\}+O(\eta^2)=\sum_{j\neq k}\eta_j(1-\alpha_j)^2+O(\eta^2),
\label{eq:app-leak-bracket}
\end{equation}
using the algebraic identity $2(1-\alpha)-(1-\alpha^2)=(1-\alpha)^2$. Substituting~(\ref{eq:app-leak-bracket}) into~(\ref{eq:app-leak-xHx-step1}),
\begin{equation}
\xi_*^\top H_*\xi_*=1-2\sum_{j\neq k}\eta_j(1-\alpha_j)-\beta\sum_{j\neq k}\eta_j(1-\alpha_j)^2+O(\eta^2).
\label{eq:app-leak-xHx}
\end{equation}

\smallskip Substituting~(\ref{eq:app-leak-norm}),~(\ref{eq:app-leak-tr}), and~(\ref{eq:app-leak-xHx}) into~(\ref{eq:app-delta-avg}), at order $\eta$,
\begin{align}
\|\xi_*\|^2\,\mathrm{tr}\,H_*-\xi_*^\top H_*\xi_*
&=(d-1)-2(\beta+d-1)\sum_{j\neq k}\eta_j(1-\alpha_j)+\beta\sum_{j\neq k}\eta_j(1-\alpha_j)^2+O(\eta^2)\nonumber\\
&=(d-1)-\sum_{j\neq k}\eta_j\Big\{2(d-1)(1-\alpha_j)+\beta(1-\alpha_j)\big[2-(1-\alpha_j)\big]\Big\}+O(\eta^2)\nonumber\\
&=(d-1)-\sum_{j\neq k}\eta_j\big\{2(d-1)(1-\alpha_j)+\beta(1-\alpha_j^2)\big\}+O(\eta^2),
\label{eq:app-leak-trick}
\end{align}
where the last line uses the identity $(1-\alpha)\,\big[2-(1-\alpha)\big]=(1-\alpha)(1+\alpha)=1-\alpha^2$. Therefore
\begin{equation}
\Delta E_*^{\mathrm{fp}}=\tfrac{\varepsilon^2}{4}\Big[(d-1)-\sum_{j\neq k}\eta_j\,\Lambda(\alpha_j)\Big]+O(\varepsilon^2\eta^2+\varepsilon^3),\quad \Lambda(\alpha):=2(d-1)(1-\alpha)+\beta(1-\alpha^2).
\label{eq:app-delta-eta}
\end{equation}
Each competitor at cosine $\alpha_j$ shaves $\eta_j\Lambda(\alpha_j)$ off the bare $(d-1)$. The geometric weight $\Lambda(\alpha)$ vanishes at $\alpha=1$ (perfect alignment, no help) and is maximal at $\alpha=0$ (orthogonal competitor, maximum reduction).

\subsection{Specialization and proof of Theorem~\ref{thm:forgetting}}
\label{app:DE-comparison}
This subsection applies the Hessian formula to the two Task-1 fixed points and then compares the resulting energy rises.

\paragraph{Outlier-like fixed point.}\label{app:DE-out}
For the outlier-like fixed point all $N$ competitors are orthogonal: $\alpha_j=0$ and $\eta_j=u=e^{-\beta}(1+o(1))$ by Proposition~\ref{prop:app-fp-out}. Substituting into~(\ref{eq:app-delta-eta}),
\begin{equation}
\Delta E_o^{\mathrm{fp}}=\tfrac{\varepsilon^2}{4}\big[(d-1)-N\,e^{-\beta}\{2(d-1)+\beta\}\big]+O(\varepsilon^2 e^{-2\beta}+\varepsilon^3).
\label{eq:app-DE-out}
\end{equation}

\paragraph{Cluster-like fixed point.}\label{app:DE-cl}
For the cluster-like fixed point near $x_1$ there are two competitor types (Proposition~\ref{prop:app-fp-cl}): the other $N-1$ cluster members with $\alpha=c$, $\eta=b=e^{-\beta(1-c)}(1+o(1))$, and the outlier with $\alpha=0$, $\eta=q=e^{-\beta}(1+o(1))$. Substituting into~(\ref{eq:app-delta-eta}),
\begin{equation}
\begin{aligned}
\Delta E_{\mathrm{cl}}^{\mathrm{fp}}
&=\tfrac{\varepsilon^2}{4}\Big[(d-1)-(N-1)\,e^{-\beta(1-c)}\{2(d-1)(1-c)+\beta(1-c^2)\}\\
&\hspace{3.0cm}-e^{-\beta}\{2(d-1)+\beta\}\Big]+O(\varepsilon^2\,e^{-2\beta(1-c)}+\varepsilon^3).
\end{aligned}
\label{eq:app-DE-cl}
\end{equation}

\paragraph{Comparison.}
Subtracting~(\ref{eq:app-DE-cl}) from~(\ref{eq:app-DE-out}) with $\Lambda(\alpha)=2(d-1)(1-\alpha)+\beta(1-\alpha^2)$,
\begin{equation}
\Delta E_o^{\mathrm{fp}}-\Delta E_{\mathrm{cl}}^{\mathrm{fp}}=\tfrac{\varepsilon^2}{4}(N-1)\big[e^{-\beta(1-c)}\,\Lambda(c)-e^{-\beta}\,\Lambda(0)\big]+o\big(\varepsilon^2\,e^{-\beta(1-c)}\big).
\label{eq:app-DE-diff}
\end{equation}
For $0<c<1$ and $\beta\gg\log N$ we have $e^{-\beta(1-c)}/e^{-\beta}=e^{\beta c}\to\infty$, while $\Lambda(c)=2(d-1)(1-c)+\beta(1-c^2)>0$. Hence the right-hand side is strictly positive at leading order, proving Theorem~\ref{thm:forgetting}.\hfill$\square$

\begin{remark}[Edge cases]
If $c=0$, the cluster degenerates into $N$ mutually orthogonal points and both fixed points collapse: $e^{-\beta(1-c)}=e^{-\beta}$, the bracket in~(\ref{eq:app-DE-diff}) vanishes, and the inequality becomes an equality at leading order. If $c<0$, the cluster members are anti-correlated, $\Lambda(c)$ may become negative, and the sign of~(\ref{eq:app-DE-diff}) can reverse. Hence the inequality is genuine to the cluster regime $0<c<1$.
\end{remark}

\section{Energy--Sharpness Relation}
\label{app:energy-sharpness}

In the main text, we describe isolated memories as forming high-energy and sharp basins, and cluster-supported memories as forming low-energy and broad basins. 
This appendix makes that statement precise in the equal-angular cluster-plus-outlier model. 
The key point is that both the local energy level and the local sharpness of a fixed point are controlled by the same off-memory softmax weights.

\subsection{Definition of local sharpness}

Let $\xi_*$ be a Task-1 fixed point of the Hopfield energy
\[
E(\xi\mid X)
=
-\frac{1}{\beta}
\log\sum_i \exp(\beta x_i^\top \xi)
+
\frac12\|\xi\|^2 .
\]
The Hessian at $\xi_*$ is
\[
H_*
=
\nabla^2_\xi E(\xi_*\mid X)
=
I
-
\beta
\left(
\sum_i p_i^* x_i x_i^\top
-
\xi_*\xi_*^\top
\right),
\]
where $p_i^*=p_i(\xi_*\mid X)$.
We define the local sharpness of the basin by the average Hessian trace in random tangential directions induced by an infinitesimal rotation,
\begin{equation}
\mathcal{S}_*
:=
\|\xi_*\|^2 \operatorname{tr} H_*
-
\xi_*^\top H_* \xi_* .
\label{eq:app-sharpness-def}
\end{equation}
Indeed, for $V=\exp(\varepsilon\Omega)$ with $\Omega=(W-W^\top)/2$ and $W_{ij}\sim\mathcal{N}(0,1)$, one has
\[
\mathbb{E}_{\Omega}
\left[
(\Omega\xi_*)^\top H_* (\Omega\xi_*)
\right]
=
\frac12
\left(
\|\xi_*\|^2 \operatorname{tr}H_*
-
\xi_*^\top H_*\xi_*
\right)
=
\frac12\mathcal{S}_* .
\]
Thus $\mathcal{S}_*$ is the curvature factor that appears in the rotation-induced energy increase:
\[
\mathbb{E}_{\Omega}
\left[
E(V^\top\xi_*\mid X)-E(\xi_*\mid X)
\right]
=
\frac{\varepsilon^2}{4}\mathcal{S}_*
+
O(\varepsilon^3).
\]
Larger $\mathcal{S}_*$ means that the basin rises more rapidly under small displacements, and we therefore use it as the local sharpness of the Hopfield basin.

\subsection{First-order expansion of energy and sharpness}

Consider a fixed point concentrated near a stored memory $x_k$. 
Write the softmax weights as
\[
p_k^*
=
1-\sum_{j\neq k}\eta_j,
\qquad
p_j^*
=
\eta_j,
\qquad
\eta_j\ll 1,
\]
and denote the cosine similarity between the competitor $x_j$ and the dominant memory $x_k$ by
\[
\alpha_j:=x_j^\top x_k .
\]
The fixed-point condition gives
\[
\xi_*
=
x_k+\sum_{j\neq k}\eta_j(x_j-x_k)
+
O(\eta^2).
\]

First, the energy at the fixed point has the expansion
\begin{equation}
E(\xi_*\mid X)
=
-\frac12
-
\frac{1}{\beta}
\sum_{j\neq k}\eta_j
+
O(\eta^2).
\label{eq:app-energy-leakage}
\end{equation}
Thus a fixed point supported by more nearby memories has lower energy. 
In this sense, the total leakage $\sum_{j\neq k}\eta_j$ measures how much additional support the dominant memory receives from neighboring memories.

Second, the sharpness defined in Eq.~\eqref{eq:app-sharpness-def} satisfies
\begin{equation}
\mathcal{S}_*
=
(d-1)
-
\sum_{j\neq k}\eta_j \Lambda(\alpha_j)
+
O(\eta^2),
\label{eq:app-sharpness-leakage}
\end{equation}
where
\begin{equation}
\Lambda(\alpha)
:=
2(d-1)(1-\alpha)
+
\beta(1-\alpha^2).
\label{eq:app-lambda-sharpness}
\end{equation}
Since $\Lambda(\alpha)>0$ for $\alpha<1$, leakage from neighboring memories lowers the local curvature. 
Therefore the same off-memory softmax mass that lowers the energy also flattens the basin. 
This is the precise local sense in which low-energy basins are broad, whereas high-energy basins are sharp.

\subsection{Cluster-plus-outlier specialization}

We now apply Eqs.~\eqref{eq:app-energy-leakage} and~\eqref{eq:app-sharpness-leakage} to the equal-angular configuration of Section~\ref{sec:theory1}. 
For the outlier fixed point, all $N$ competitors are orthogonal, so $\alpha=0$ and $\eta=e^{-\beta}(1+o(1))$. 
Therefore,
\begin{equation}
E_o^*
=
-\frac12
-
\frac{N}{\beta}e^{-\beta}
+
o(e^{-\beta}/\beta),
\label{eq:app-energy-outlier}
\end{equation}
and
\begin{equation}
\mathcal{S}_o
=
(d-1)
-
N e^{-\beta}\Lambda(0)
+
o(e^{-\beta}).
\label{eq:app-sharpness-outlier}
\end{equation}

For the cluster fixed point near $x_1$, the $N-1$ other cluster memories have cosine $c$ and leakage $e^{-\beta(1-c)}$, while the outlier has cosine $0$ and leakage $e^{-\beta}$. 
Thus,
\begin{equation}
E_{\mathrm{cl}}^*
=
-\frac12
-
\frac{1}{\beta}
\left\{
(N-1)e^{-\beta(1-c)}
+
e^{-\beta}
\right\}
+
o(e^{-\beta(1-c)}/\beta),
\label{eq:app-energy-cluster}
\end{equation}
and
\begin{equation}
\mathcal{S}_{\mathrm{cl}}
=
(d-1)
-
(N-1)e^{-\beta(1-c)}\Lambda(c)
-
e^{-\beta}\Lambda(0)
+
o(e^{-\beta(1-c)}).
\label{eq:app-sharpness-cluster}
\end{equation}

Subtracting Eq.~\eqref{eq:app-energy-cluster} from Eq.~\eqref{eq:app-energy-outlier} gives
\begin{equation}
E_o^*
-
E_{\mathrm{cl}}^*
=
\frac{N-1}{\beta}
\left(
e^{-\beta(1-c)}
-
e^{-\beta}
\right)
+
o(e^{-\beta(1-c)}/\beta)
>
0,
\label{eq:app-energy-order}
\end{equation}
because $0<c<1$ implies $e^{-\beta(1-c)}\gg e^{-\beta}$ at large $\beta$. 
Hence the outlier fixed point has higher Hopfield energy than the cluster fixed point.

Likewise, subtracting Eq.~\eqref{eq:app-sharpness-cluster} from Eq.~\eqref{eq:app-sharpness-outlier} gives
\begin{equation}
\mathcal{S}_o
-
\mathcal{S}_{\mathrm{cl}}
=
(N-1)
\left[
e^{-\beta(1-c)}\Lambda(c)
-
e^{-\beta}\Lambda(0)
\right]
+
o(e^{-\beta(1-c)})
>
0 .
\label{eq:app-sharpness-order}
\end{equation}
Thus the same outlier fixed point is also sharper. 
Equations~\eqref{eq:app-energy-order} and~\eqref{eq:app-sharpness-order} justify the shorthand used in the main text: in this model, high-energy memories correspond to sharp isolated basins, whereas low-energy memories correspond to broad cluster-supported basins.


\subsection{Sharpness in diffusion-model experiments}
\label{app:energy-sharpness-experiment}

We further examine the relation between Hopfield energy, local sharpness, and reconstruction-based forgetting in the diffusion-model experiments. 
Here we focus on CIFAR-10 and compute all quantities sample-wise for Task-1 images.

For each Task-1 sample $x$, we first compute its Hopfield energy with respect to the Task-1 reference set $\mathcal{D}_1$,
\[
E_t(x\mid \mathcal{D}_1),
\]
using the same energy computation as in the main diffusion-model experiments. 
To quantify local sharpness, we do not explicitly form the Hessian. 
Instead, we use a perturbation-based finite-difference estimate: for each sample $x$, we draw $M$ small random perturbations $\{\delta_m\}_{m=1}^M$ and measure how much the Hopfield energy increases around $x$,
\begin{equation}
\widehat{\mathcal{S}}_t(x)
=
\frac{1}{M}
\sum_{m=1}^{M}
\left[
E_t(x+\delta_m\mid\mathcal{D}_1)
-
E_t(x\mid\mathcal{D}_1)
\right],
\label{eq:perturbation-sharpness}
\end{equation}
up to the fixed perturbation-scale normalization used in the implementation.
In the experiments, we use $M=100$ random perturbations for each sample. 
Thus, $\widehat{\mathcal{S}}_t(x)$ measures how steeply the Hopfield energy changes under local random perturbations: larger values indicate that the sample lies in a sharper local region of the energy landscape.

Figure~\ref{fig:diffusion-energy-sharpness} summarizes the result. 
First, we compare the perturbation-based sharpness $\widehat{\mathcal{S}}_t(x)$ with the Hopfield energy $E_t(x\mid\mathcal{D}_1)$. 
Higher-energy samples tend to have larger perturbation sharpness, consistent with the theoretical picture that weakly supported samples lie in sharper regions of the landscape. 
Second, we compare $\widehat{\mathcal{S}}_t(x)$ with the reconstruction-based forgetting metric $\mathcal{F}_{t^\star}(x)$ defined in Eq.~\eqref{eq:reconstruction-forgetting}. 
Samples with larger perturbation sharpness also tend to show larger reconstruction error after Task-2 training. 
These results suggest that the energy-based ordering observed in the main text is accompanied by a corresponding difference in local sharpness, and that this difference is reflected in reconstruction behavior after continual learning.

\begin{figure}[t]
\centering
\begin{subfigure}[b]{0.45\linewidth}
\centering
\includegraphics[width=\linewidth]{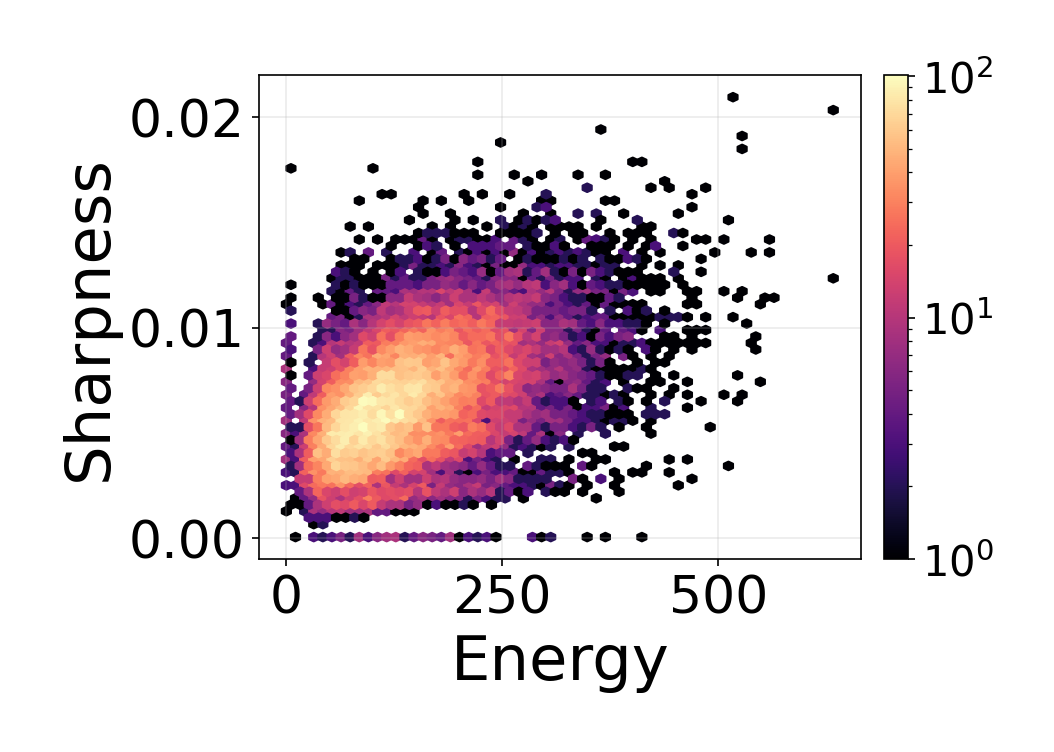}
\caption{Energy vs. perturbation sharpness.}
\label{fig:cifar-energy-sharpness}
\end{subfigure}
\hfill
\begin{subfigure}[b]{0.45\linewidth}
\centering
\includegraphics[width=\linewidth]{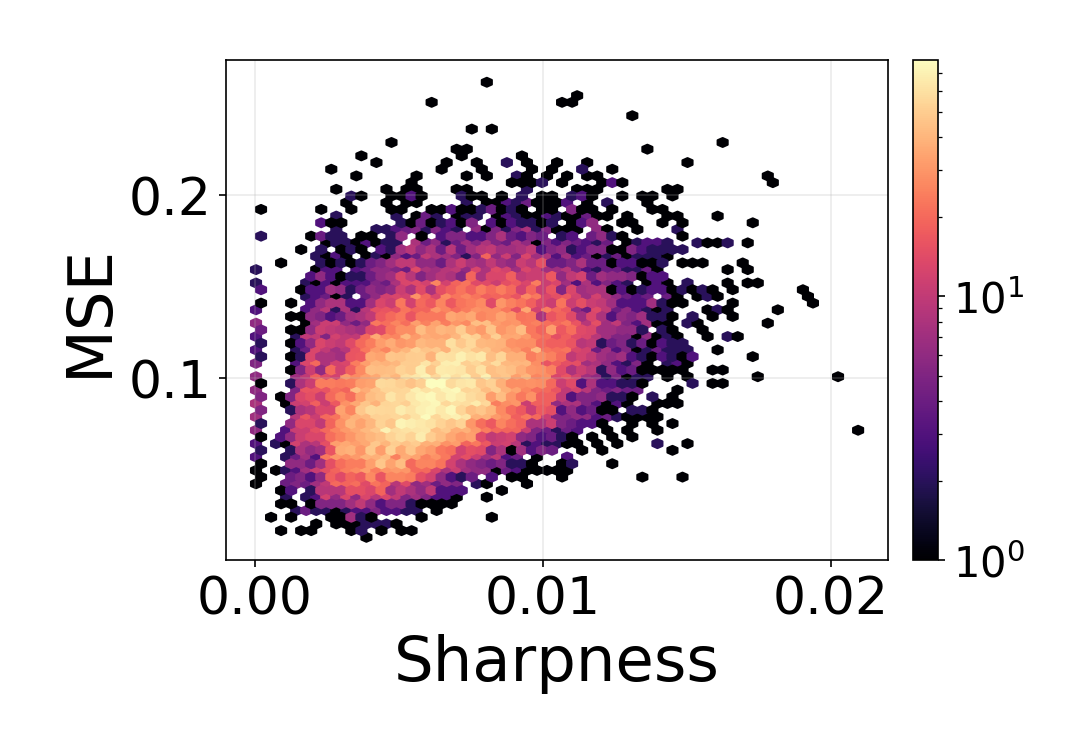}
\caption{Perturbation sharpness vs. reconstruction forgetting.}
\label{fig:cifar-sharpness-reconstruction}
\end{subfigure}
\caption{
Perturbation-based sharpness in the diffusion-model experiment on CIFAR-10.
(a) Higher-Hopfield-energy samples tend to have larger perturbation sharpness $\widehat{\mathcal{S}}_t(x)$.
(b) Samples with larger perturbation sharpness tend to have larger reconstruction-based forgetting $\mathcal{F}_{t^\star}(x)$ after Task-2 training.
}
\label{fig:diffusion-energy-sharpness}
\end{figure}

\section{Full derivation of Theorem~\ref{thm:replay}}
\label{app:replay-proof}

We derive the exact replay identity, isolate the duplicate-memory baseline, expand the baseline-subtracted gain to leading order in the small rotation, evaluate the five canonical gains, and prove the ordering of Theorem~\ref{thm:replay}.

\subsection{Exact replay identity}
\label{app:replay-identity}

Let $\mathcal{I}=\{1,\dots,N,o\}$, $X=\{x_\mu:\mu\in\mathcal{I}\}$, and $X_2(V)=VX$. When a replay point $x_r\in X$ is added, write $X_2^{+r}(V):=X_2(V)\cup\{x_r\}$, with the replayed sample counted as an additional term in the log-sum-exp. For a Task-1 fixed point $\xi_k^*$, the  energy rise without replay was already defined in Section~\ref{app:Hessian} as
\begin{equation}
\Delta E_k^{\mathrm{fp}}(V)=E(\xi_k^*\mid VX)-E(\xi_k^*\mid X),\qquad \Delta E_k^{\mathrm{fp}}(I)=0.
\label{eq:app-DEfp-recall}
\end{equation}
We additionally need the Task-1 softmax weight of the replay point at $\xi_k^*$,
\begin{equation}
p_{r\mid k}:=\frac{e^{\beta x_r^\top\xi_k^*}}{\sum_{\mu\in\mathcal{I}}e^{\beta x_\mu^\top\xi_k^*}}.
\label{eq:app-p-rk-def}
\end{equation}
The raw replay-induced energy decrease is
\begin{equation}
\mathcal{R}_{r\to k}(V):=E(\xi_k^*\mid X_2(V))-E(\xi_k^*\mid X_2^{+r}(V)).
\label{eq:app-raw-gain-def}
\end{equation}
Using the log-sum-exp form~(\ref{eq:app-E-T}) of $E$, the difference of the two log-sum-exp's is
\begin{equation}
\mathcal{R}_{r\to k}(V)=\tfrac{1}{\beta}\log\!\left(1+\frac{e^{\beta x_r^\top\xi_k^*}}{\sum_{\mu\in\mathcal{I}}e^{\beta(Vx_\mu)^\top\xi_k^*}}\right).
\label{eq:app-raw-gain-log}
\end{equation}
Now observe that, by the definition~(\ref{eq:app-DEfp-recall}) and the log-sum-exp form,
\begin{equation}
e^{\beta\Delta E_k^{\mathrm{fp}}(V)}=\frac{\sum_\mu e^{\beta x_\mu^\top\xi_k^*}}{\sum_\mu e^{\beta(Vx_\mu)^\top\xi_k^*}}.
\label{eq:app-DE-ratio-id}
\end{equation}
Multiplying numerator and denominator inside~(\ref{eq:app-raw-gain-log}) by $e^{\beta\Delta E_k^{\mathrm{fp}}(V)}$ converts the rotated denominator back into the Task-1 partition function:
\begin{equation}
\frac{e^{\beta x_r^\top\xi_k^*}}{\sum_\mu e^{\beta(Vx_\mu)^\top\xi_k^*}}=\frac{e^{\beta x_r^\top\xi_k^*}\,e^{\beta\Delta E_k^{\mathrm{fp}}(V)}}{\sum_\mu e^{\beta x_\mu^\top\xi_k^*}}=p_{r\mid k}\,e^{\beta\Delta E_k^{\mathrm{fp}}(V)}.
\label{eq:app-multiply-trick}
\end{equation}
Hence the exact identity
\begin{equation}
\mathcal{R}_{r\to k}(V)=\tfrac{1}{\beta}\log\!\big(1+p_{r\mid k}\,e^{\beta\Delta E_k^{\mathrm{fp}}(V)}\big).
\label{eq:app-raw-exact}
\end{equation}
At $V=I$, $\Delta E_k^{\mathrm{fp}}(I)=0$ and~(\ref{eq:app-raw-exact}) reduces to the trivial duplicate-memory baseline
\begin{equation}
\mathcal{R}_{r\to k}(I)=\tfrac{1}{\beta}\log(1+p_{r\mid k}),
\label{eq:app-baseline}
\end{equation}
which is the energy decrease one obtains merely from replicating an already-stored Task-1 memory. Subtracting this baseline and using $\log(a)-\log(b)=\log(a/b)$,
\begin{equation}
\widehat{\Delta}_{r\to k}(V):=\mathcal{R}_{r\to k}(V)-\mathcal{R}_{r\to k}(I)=\tfrac{1}{\beta}\log\!\frac{1+p_{r\mid k}\,e^{\beta\Delta E_k^{\mathrm{fp}}(V)}}{1+p_{r\mid k}}.
\label{eq:app-delta-exact}
\end{equation}
This separates the two effects: $p_{r\mid k}$ measures the relevance of the replay point (geometry), and $\Delta E_k^{\mathrm{fp}}(V)$ measures how much the rotated Task-2 dataset raises the energy at $\xi_k^*$ (drift) — exactly the same object that drives Theorem~\ref{thm:forgetting}.

\subsection{Small-rotation expansion and master formula}

Define
\begin{equation}
\rho(p):=\frac{p}{1+p}.
\label{eq:app-rho-def}
\end{equation}
This factor determines the response of forgetting mitigation to the energy rise, because
\begin{equation}
\left.
\frac{\partial}{\partial \Delta E}
\left\{
\frac{1}{\beta}\log\frac{1+p\,e^{\beta\Delta E}}{1+p}
\right\}
\right|_{\Delta E=0}
=
\rho(p).
\label{eq:app-rho-response}
\end{equation}
We therefore call $\rho(p)$ the \emph{replay susceptibility}. This name refers to the response to an energy change in Eq.~(\ref{eq:app-rho-response}). Taylor-expanding~(\ref{eq:app-delta-exact}) in $\Delta E_k^{\mathrm{fp}}(V)=O(\varepsilon^2)$,
\begin{equation}
\widehat{\Delta}_{r\to k}(V)=\rho(p_{r\mid k})\,\Delta E_k^{\mathrm{fp}}(V)+\tfrac{\beta}{2}\rho(p_{r\mid k})(1-\rho(p_{r\mid k}))\,\Delta E_k^{\mathrm{fp}}(V)^2+O\big(\beta^2\rho(p_{r\mid k})\,|\Delta E_k^{\mathrm{fp}}(V)|^3\big).
\label{eq:app-delta-expand}
\end{equation}
Averaging over $V$ and setting $\Delta_{r\to k}:=\langle\widehat{\Delta}_{r\to k}(V)\rangle_V$, and recalling $\Delta E_k^{\mathrm{fp}}:=\langle\Delta E_k^{\mathrm{fp}}(V)\rangle_V$ from~(\ref{eq:DEfp-def}),
\begin{equation}
\Delta_{r\to k}=\rho(p_{r\mid k})\,\Delta E_k^{\mathrm{fp}}+O\big(\beta\rho(p_{r\mid k})\varepsilon^4\big).
\label{eq:app-master}
\end{equation}
The factor $\Delta E_k^{\mathrm{fp}}$ is exactly the rotation-averaged energy rise computed in Section~\ref{app:DE-comparison} (Eqs.~(\ref{eq:app-DE-out}),(\ref{eq:app-DE-cl})).

\subsection{Susceptibility asymptotics}
From Appendix~\ref{app:forgetting-proof},
\[
v=p_{o\mid o}
=
1-N e^{-\beta}+O(\beta e^{-2\beta}),
\]
\[
a=p_{1\mid1}
=
1-(N-1)e^{-\beta(1-c)}-e^{-\beta}
+o(e^{-\beta(1-c)}),
\]
\[
b=p_{1\mid2}
=
e^{-\beta(1-c)}(1+o(1)),
\qquad
u=p_{1\mid o}
=
e^{-\beta}(1+o(1)),
\qquad
q=p_{o\mid1}
=
e^{-\beta}(1+o(1)).
\]
Applying \(\rho(p)=p/(1+p)\),
\begin{equation}
\begin{split}
&\rho(v)=\tfrac12-\tfrac{N\,e^{-\beta}}{4}+O(\beta\,e^{-2\beta}),\\
&\rho(a)=\tfrac12-\tfrac{(N-1)\,e^{-\beta(1-c)}+e^{-\beta}}{4}
+o(e^{-\beta(1-c)}),\\
&\rho(b)=e^{-\beta(1-c)}(1+o(1)),\qquad
\rho(u)=e^{-\beta}(1+o(1)),\qquad
\rho(q)=e^{-\beta}(1+o(1)).
\label{eq:app-rho-asym}
\end{split}
\end{equation}

Hence
\[
\rho(v)>\rho(a)\gg\rho(b)\gg\rho(u)\asymp\rho(q).
\]


\subsection{Five replay gains}
\label{app:four-gains}
Recall from~(\ref{eq:app-DE-out})--(\ref{eq:app-DE-cl}), with $\Lambda(\alpha)=2(d-1)(1-\alpha)+\beta(1-\alpha^2)$,
\begin{align}
\Delta E_o^{\mathrm{fp}}&=\tfrac{\varepsilon^2}{4}\big[(d-1)-N\,e^{-\beta}\,\Lambda(0)\big]+O(\varepsilon^2\,e^{-2\beta}+\varepsilon^3),\\
\Delta E_{\mathrm{cl}}^{\mathrm{fp}}&=\tfrac{\varepsilon^2}{4}\big[(d-1)-(N-1)\,e^{-\beta(1-c)}\,\Lambda(c)-e^{-\beta}\,\Lambda(0)\big]+O(\varepsilon^2\,e^{-2\beta(1-c)}+\varepsilon^3),\\
\Delta E_o^{\mathrm{fp}}-\Delta E_{\mathrm{cl}}^{\mathrm{fp}}&=\tfrac{\varepsilon^2}{4}(N-1)\big[e^{-\beta(1-c)}\,\Lambda(c)-e^{-\beta}\,\Lambda(0)\big]+o\big(\varepsilon^2\,e^{-\beta(1-c)}\big)>0.
\end{align}
Multiplying~(\ref{eq:app-rho-asym}) by $\Delta E_o^{\mathrm{fp}}$ or $\Delta E_{\mathrm{cl}}^{\mathrm{fp}}$ via the master formula~(\ref{eq:app-master}) yields
\begin{align}
\Delta_{o \to o}
&=
\tfrac12\,\Delta E_o^{\mathrm{fp}}\,(1+o(1))
=
\tfrac{\varepsilon^2}{8}(d-1)
+
O\big(\varepsilon^2 e^{-\beta}+\varepsilon^3+\beta\varepsilon^4\big),
\\
\Delta_{\mathrm{cl}_1 \to \mathrm{cl}_1}
&=
\tfrac12\,\Delta E_{\mathrm{cl}}^{\mathrm{fp}}\,(1+o(1))
=
\tfrac{\varepsilon^2}{8}(d-1)
+
O\big(\varepsilon^2 e^{-\beta(1-c)}+\varepsilon^3+\beta\varepsilon^4\big),
\\
\Delta_{\mathrm{cl}_1\to \mathrm{cl}_2}
&=
e^{-\beta(1-c)}\,\Delta E_{\mathrm{cl}}^{\mathrm{fp}}\,(1+o(1))
=
\tfrac{\varepsilon^2}{4}(d-1)e^{-\beta(1-c)}(1+o(1)),
\\
\Delta_{\mathrm{cl}_1\to o}
&=
e^{-\beta}\,\Delta E_o^{\mathrm{fp}}\,(1+o(1))
=
\tfrac{\varepsilon^2}{4}(d-1)e^{-\beta}(1+o(1)),
\\
\Delta_{o\to \mathrm{cl}_1}
&=
e^{-\beta}\,\Delta E_{\mathrm{cl}}^{\mathrm{fp}}\,(1+o(1))
=
\tfrac{\varepsilon^2}{4}(d-1)e^{-\beta}(1+o(1)).
\label{eq:app-five-gains}
\end{align}

\subsection{Ordering}
\begin{enumerate}
\item \emph{\(\Delta_{o\to o}>\Delta_{\mathrm{cl}_1\to \mathrm{cl}_1}\).}
\[
\Delta_{o\to o}-\Delta_{\mathrm{cl}_1\to \mathrm{cl}_1}
=
\{\rho(v)-\rho(a)\}\Delta E_{\mathrm{cl}}^{\mathrm{fp}}
+
\rho(v)\{\Delta E_o^{\mathrm{fp}}-\Delta E_{\mathrm{cl}}^{\mathrm{fp}}\}
+
o\big(\varepsilon^2e^{-\beta(1-c)}\big).
\]
From~(\ref{eq:app-rho-asym}),
\(\rho(v)-\rho(a)
=
\tfrac{(N-1)(e^{-\beta(1-c)}-e^{-\beta})}{4}
+
o(e^{-\beta(1-c)})>0\),
and Theorem~\ref{thm:forgetting} gives
\(\Delta E_o^{\mathrm{fp}}>\Delta E_{\mathrm{cl}}^{\mathrm{fp}}>0\).
Thus \(\Delta_{o\to o}>\Delta_{\mathrm{cl}_1\to \mathrm{cl}_1}\).

\item \emph{\(\Delta_{\mathrm{cl}_1\to \mathrm{cl}_1}\gg\Delta_{\mathrm{cl}_1\to \mathrm{cl}_2}\).}
\[
\frac{\Delta_{\mathrm{cl}_1\to \mathrm{cl}_2}}{\Delta_{\mathrm{cl}_1\to \mathrm{cl}_1}}
=
\frac{\rho(b)}{\rho(a)}(1+o(1))
=
2e^{-\beta(1-c)}(1+o(1))
\to0.
\]

\item \emph{\(\Delta_{\mathrm{cl}_1\to \mathrm{cl}_2}\gg\Delta_{\mathrm{cl}_1\to o}\).}
\[
\frac{\Delta_{\mathrm{cl}_1\to o}}{\Delta_{\mathrm{cl}_1\to \mathrm{cl}_2}}
=
\frac{\rho(u)}{\rho(b)}
\frac{\Delta E_o^{\mathrm{fp}}}{\Delta E_{\mathrm{cl}}^{\mathrm{fp}}}
=
e^{-\beta c}(1+o(1))
\to0.
\]

\item \emph{\(\Delta_{\mathrm{cl}_1\to \mathrm{cl}_2}\gg\Delta_{o\to \mathrm{cl}_1}\).}
\[
\frac{\Delta_{o\to \mathrm{cl}_1}}{\Delta_{\mathrm{cl}_1\to \mathrm{cl}_2}}
=
\frac{\rho(q)}{\rho(b)}
=
e^{-\beta c}(1+o(1))
\to0.
\]
Moreover,
\[
\frac{\Delta_{o\to \mathrm{cl}_1}}{\Delta_{\mathrm{cl}_1\to o}}
=
\frac{\rho(q)}{\rho(u)}
\frac{\Delta E_{\mathrm{cl}}^{\mathrm{fp}}}{\Delta E_o^{\mathrm{fp}}}
=
1+o(1),
\]

\end{enumerate}

so \(\Delta_{o\to \mathrm{cl}_1}\) and \(\Delta_{\mathrm{cl}_1\to o}\) have the same exponential order.
Combining the four comparisons gives
\[
\Delta_{o\to o}>\Delta_{\mathrm{cl}_1\to \mathrm{cl}_1}
\gg
\Delta_{\mathrm{cl}_1\to \mathrm{cl}_2}
\gg
\{\Delta_{\mathrm{cl}_1\to o},\Delta_{o\to \mathrm{cl}_1}\},
\qquad
\Delta_{\mathrm{cl}_1\to o}\asymp\Delta_{o\to \mathrm{cl}_1},
\]
which proves Theorem~\ref{thm:replay}.
\hfill\(\square\)

\begin{remark}
The subtraction of $\mathcal{R}_{r\to k}(I)=\tfrac{1}{\beta}\log(1+p_{r\mid k})$ removes the trivial gain from duplicating an already-stored memory. What is left is the fraction $\rho(p_{r\mid k})=p_{r\mid k}/(1+p_{r\mid k})$ of the rotation-induced rise $\Delta E_k^{\mathrm{fp}}(V)$ that the replay point can recover. Self-replay ($p\to 1$) recovers $\tfrac12$; within-cluster cross-replay recovers $e^{-\beta(1-c)}$; cross-type replay between cluster and outlier recovers $e^{-\beta}$. The ordering then follows from $\tfrac12\gg e^{-\beta(1-c)}\gg e^{-\beta}$ and $\Delta E_o^{\mathrm{fp}}>\Delta E_{\mathrm{cl}}^{\mathrm{fp}}$.
\end{remark}

\section{Synthetic experiments on MHNs}
\label{app:toy-experiments}

\subsection{Intrinsic forgetting in synthetic MHNs}
\label{app:toy-forgetting}

To complement the theory in Section~\ref{sec:theory1}, we instantiate the equal-angular MHN at moderate finite values ($N=12$, $d=50$, $c=0.35$, $\beta=8$, $\varepsilon=0.02$, 2000 random rotations) -- well outside the asymptotic regime $\beta\gg\log N$, $\varepsilon\ll 1$ used in the proof -- and compute $\Delta E^{\mathrm{fp}}$ in closed form. Even in this finite-parameter regime the outlier fixed point has a strictly larger rotation-induced energy rise than the cluster fixed points, matching the qualitative ordering predicted by Theorem~\ref{thm:forgetting}. The corresponding learned MHN-BM result is shown in the main text in Fig.~\ref{fig:fg-attnbm}.

\begin{figure}[t]
\centering
\includegraphics[width=0.75\linewidth]{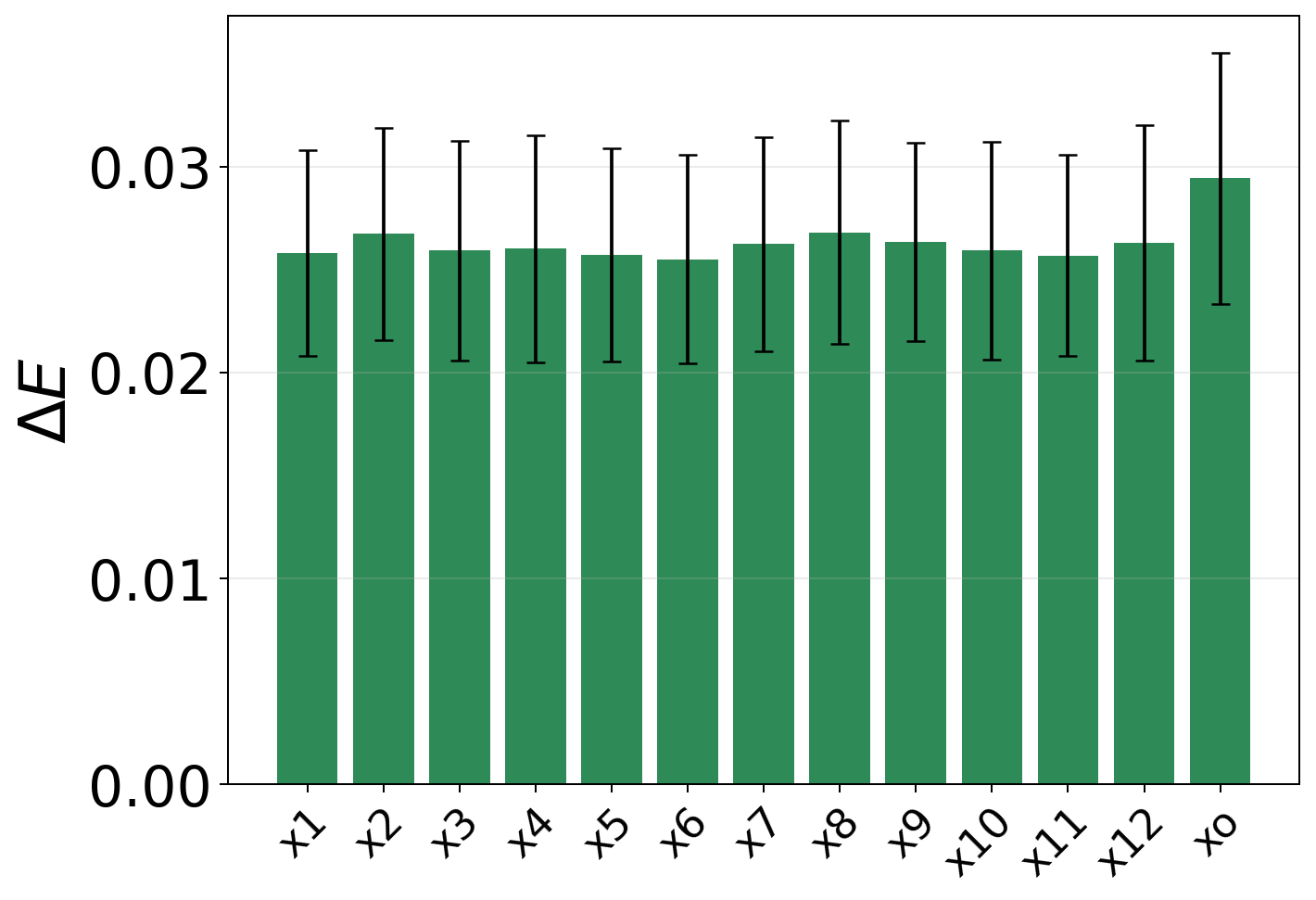}
\label{fig:fg-mhn}
\caption{Closed-form MHN corroboration of Theorem~\ref{thm:forgetting}. The outlier fixed point exhibits a larger rotation-induced energy rise than the cluster fixed points.}
\label{fig:toy-forgetting}
\end{figure}

\subsection{Replay ordering in synthetic MHNs}
\label{app:toy-replay}

For completeness, we also verify Theorem~\ref{thm:replay} in the exact equal-angular regime it was proved in. In the closed-form MHN, the full replay-gain $\{\Delta_{r\to k}\}_{r,k}$ shows the predicted ordering: the outlier self-entry $\Delta_{o\to o}$ is largest, cluster self-gains $\Delta_{\mathrm{cl}_1\to \mathrm{cl}_1}$ are next, within-cluster cross entries $\Delta_{\mathrm{cl}_1\to \mathrm{cl}_2}$ are substantially smaller, and cross-type entries $\Delta_{\mathrm{cl}_1\to o},\Delta_{o\to \mathrm{cl}_1}$ are smallest. The corresponding learned MHN-BM corroboration is shown in the main text in Fig.~\ref{fig:rp-attnbm}.

\begin{figure}[t]
\centering
\includegraphics[width=0.5\linewidth]{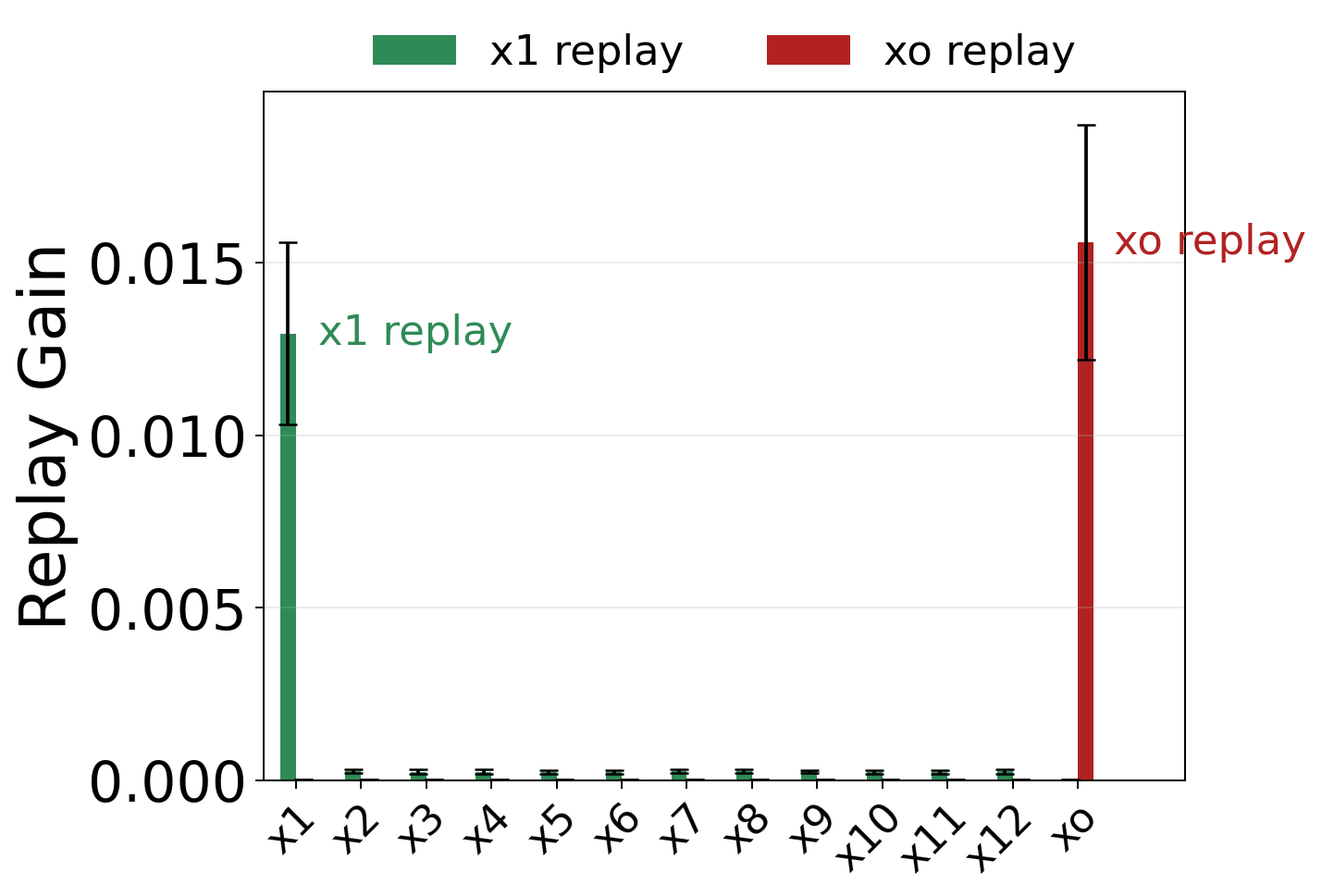}
\label{fig:rp-mhn}
\caption{Closed-form MHN corroboration of the replay ordering in Theorem~\ref{thm:replay}. The replay-gain reproduces the ranking $\Delta_{o\to o}>\Delta_{\mathrm{cl}_1\to \mathrm{cl}_1}\gg\Delta_{\mathrm{cl}_1\to \mathrm{cl}_2}\gg\{\Delta_{\mathrm{cl}_1\to o},\Delta_{o\to \mathrm{cl}_1}\}$.}
\label{fig:replay-toy}
\end{figure}

\section{Domain-incremental learning experiment: Rotated MNIST}
\label{app:rotated-mnist}

The theoretical analysis of Sections~\ref{sec:theory1} and~\ref{sec:theory2} models continual learning as a small rotation $V$ of the dataset, which corresponds most directly to a \emph{domain-incremental} setting in which the same content is presented under a perturbation of the input distribution. The main-text experiments focus on the more widely used \emph{incremental-class-learning} benchmark with split CIFAR-10. To complement them, we report here the corresponding domain-incremental experiment on rotated MNIST -- the setup that most literally matches the rotated-dataset assumption underlying our theorems.

\paragraph{Setup.}
Task~1 is standard MNIST. Task~2 is the same MNIST images rotated by $30^\circ$, instantiating the rotated-dataset construction of Section~\ref{sec:theory1} on real images. We train a pixel-space DDPM from scratch on Task~1 and continue training on Task~2. After Task~1 we assign every Task-1 image a scalar Hopfield energy $E_\theta(x)$ via the log-sum-exp expression~(\ref{eq:mhn-energy}), using all Task-1 training images as stored memories. As in the main-text diffusion-model experiments, forgetting is measured by the per-sample reconstruction MSE $\mathcal{F}_{t^\star}(x)$ defined in Eq.~(\ref{eq:reconstruction-forgetting}). For the replay analysis we use the same three policies as in Section~\ref{sec:exp-replay} -- \emph{no replay} (baseline), energy-guided \emph{top-$K$} (high-energy buffer), and energy-guided \emph{bottom-$K$} (low-energy buffer) -- with an identical buffer size $K=5000$. Full hyperparameters are listed in Appendix~\ref{app:hyperparams}.

\paragraph{Forgetting indicator (Fig.~\ref{fig:rmnist-forgetting}).}
After Task~2 fine-tuning, per-sample reconstruction MSE is a monotone-increasing function of the Task-1 Hopfield energy, reproducing the same qualitative relationship observed for split CIFAR-10 in Fig.~\ref{fig:fg-sd} and Fig.~\ref{fig:fg-pix}. Higher-energy Task-1 images suffer larger reconstruction forgetting, exactly as predicted by Theorem~\ref{thm:forgetting}.

\begin{figure}[t]
\centering
\includegraphics[width=0.5\linewidth]{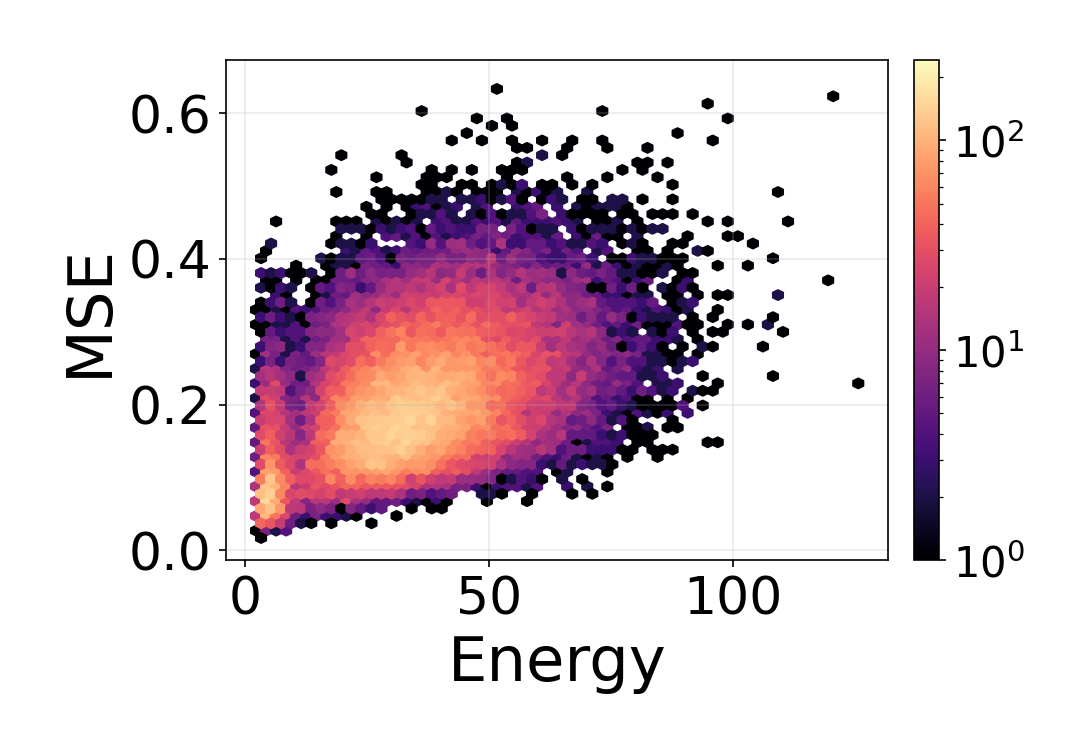}
\caption{Rotated MNIST (domain-incremental learning): per-sample reconstruction MSE $\mathcal{F}_{t^\star}(x)$ after Task~2 training versus internal Hopfield energy measured on Task~1, without replay. The same monotone energy--forgetting relation observed for split CIFAR-10 in Fig.~\ref{fig:forgetting} is recovered.}
\label{fig:rmnist-forgetting}
\end{figure}

\paragraph{Energy-guided replay (Fig.~\ref{fig:rmnist-replay}).}
The mean--range tradeoff predicted by Eq.~(\ref{eq:replay-master}) -- $\Delta_{r\to k}\sim\rho_{r\mid k}\,\Delta E_k^{\mathrm{fp}}$ -- is reproduced on rotated MNIST. The high-energy buffer (Fig.~\ref{fig:rmnist-replay}\hyperref[fig:rmnist-replay]{a}) yields a sharply peaked bin-mean reduction concentrated in the high-energy tail (Fig.~\ref{fig:rmnist-replay}\hyperref[fig:rmnist-replay]{c}), whereas the low-energy buffer (Fig.~\ref{fig:rmnist-replay}\hyperref[fig:rmnist-replay]{b}) gives a smaller per-sample reduction spread over a wider low-energy range. The bin-summed reduction (Fig.~\ref{fig:rmnist-replay}\hyperref[fig:rmnist-replay]{d}) is dominated by the high-energy tail under the high-energy buffer, exactly as in the CIFAR-10 results (Fig.~\ref{fig:replay-sd},~\ref{fig:replay-pix}). Together with the main-text experiments, this confirms that the energy-guided replay rule predicted by Theorem~\ref{thm:replay} is robust across both incremental-class and domain-incremental continual learning of diffusion models.

\begin{figure}[t]
\centering
\begin{subfigure}[b]{0.24\linewidth}
\centering
\includegraphics[width=\linewidth]{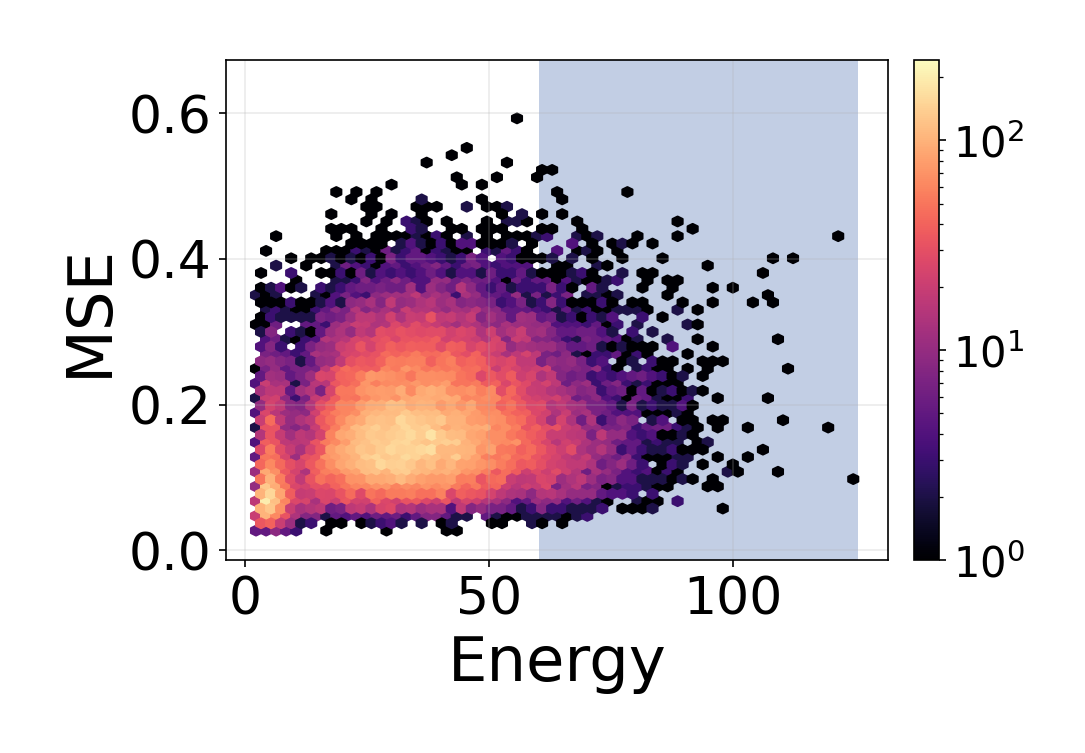}
\caption{High-energy replay.}
\label{fig:rmnist-rp-hi}
\end{subfigure}\hfill
\begin{subfigure}[b]{0.24\linewidth}
\centering
\includegraphics[width=\linewidth]{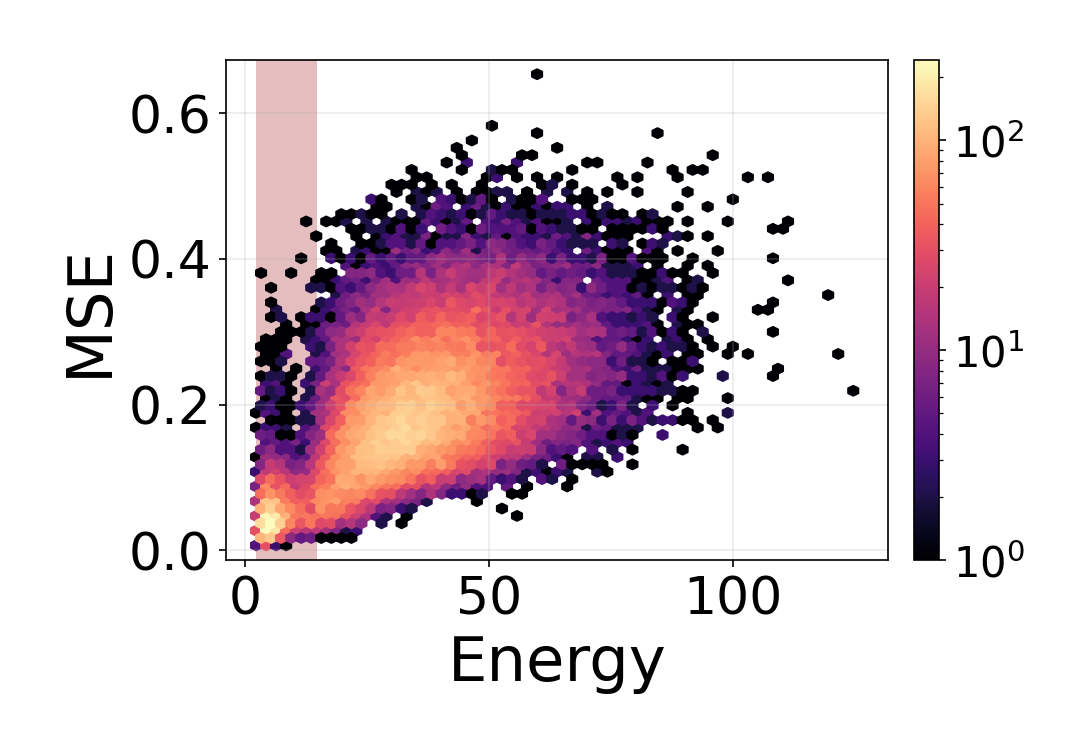}
\caption{Low-energy replay.}
\label{fig:rmnist-rp-lo}
\end{subfigure}\hfill
\begin{subfigure}[b]{0.24\linewidth}
\centering
\includegraphics[width=\linewidth]{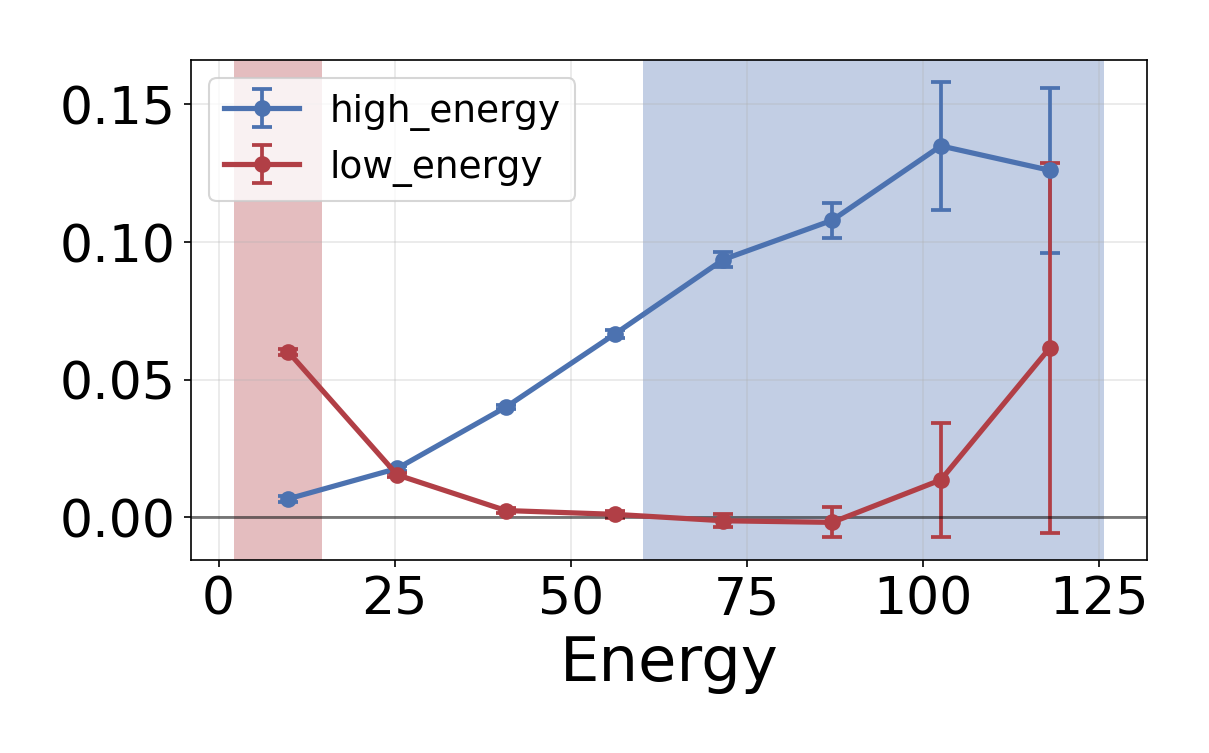}
\caption{Mean Reduction.}
\label{fig:rmnist-rp-mean}
\end{subfigure}\hfill
\begin{subfigure}[b]{0.24\linewidth}
\centering
\includegraphics[width=\linewidth]{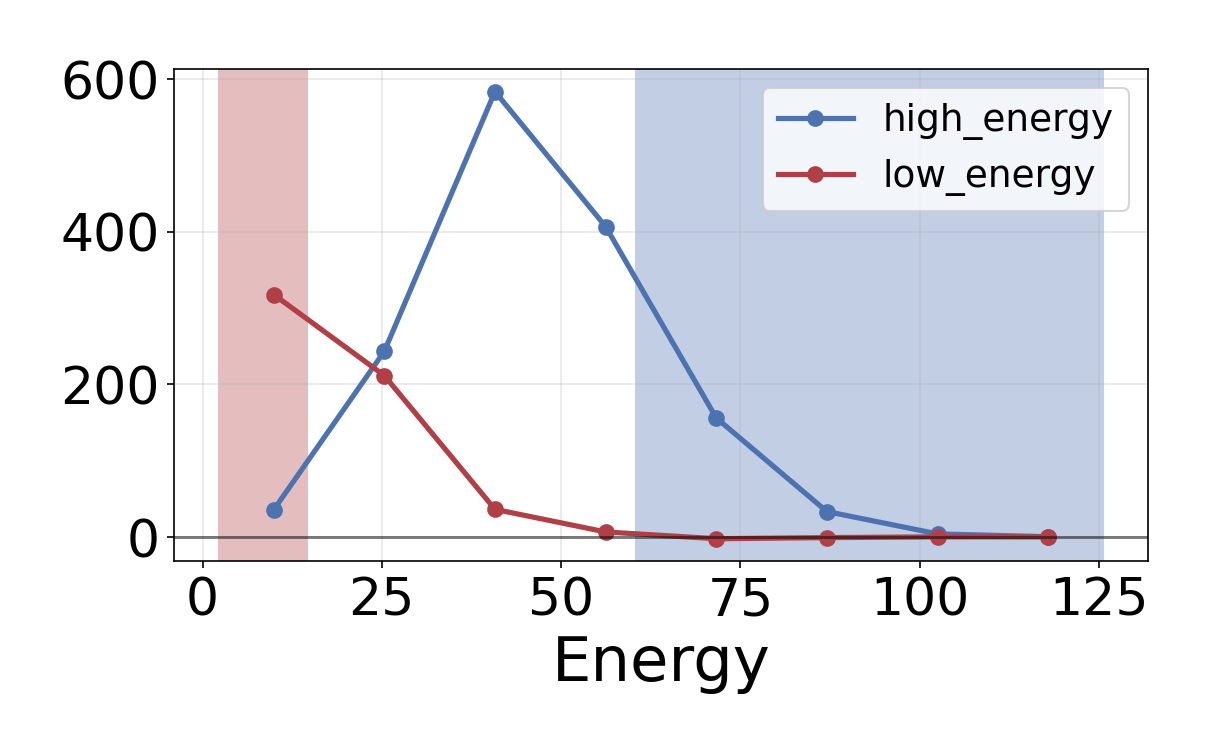}
\caption{Sum Reduction.}
\label{fig:rmnist-rp-sum}
\end{subfigure}
\caption{Energy-guided replay on rotated MNIST (domain-incremental learning). The same pattern as Fig.~\ref{fig:replay-sd} and Fig.~\ref{fig:replay-pix} is reproduced: (a) high-energy replay produces a sharply peaked, narrow rescue at high energy, (b) low-energy replay produces a smaller but broader rescue at low energy, and (c,d) the per-bin reductions follow the theoretical replay ordering in Eq.~(\ref{eq:replay-order}). The replay benefit is again dominated by the high-energy tail.}
\label{fig:rmnist-replay}
\end{figure}

\section{Reconstruction metric for diffusion-model forgetting}
\label{app:diffusion-reconstruction-metric}

The theory in Section~\ref{sec:energy-indicator} defines forgetting as an intrinsic change in the Hopfield energy,
\[
\Delta E(x)
=
E(x\mid\mathcal{X}_{\mathrm{new}})
-
E(x\mid\mathcal{X}_{\mathrm{old}}).
\]
For an explicit modern Hopfield network this quantity can be evaluated directly.
For a trained diffusion model, however, the corresponding implicit energy is not available as a normalized scalar function for each sample and each training stage.
Although the score determines the local reverse dynamics, measuring the exact sample-wise energy increase before and after Task-2 training would require reconstructing or integrating the implicit energy landscape itself.
We therefore use a reconstruction-based quantity as an operational proxy for intrinsic forgetting.

For each Task-1 image $x$, we first corrupt it to a fixed diffusion timestep $t^\star$ and then run the reverse process of the model obtained after Task-2 training.
Let $\hat{x}_{0,t^\star}$ denote the resulting reconstruction.
We define the reconstruction forgetting score as
\[
\mathcal{F}_{t^\star}(x)
=
\|x-\hat{x}_{0,t^\star}\|_2^2 .
\]
In Stable Diffusion experiments, the image is first encoded into the VAE latent space, noise is added to the latent, the UNet denoises the latent trajectory, and the final latent is decoded back to pixel space.
We report the pixel-space reconstruction MSE in the main figures.
In the pixel-space diffusion experiments, the same procedure is performed directly in pixel space.
When multiple stochastic reconstructions are used, we average the MSE over generations.

This metric should be interpreted as a proxy rather than an exact estimator of $\Delta E(x)$.
If Task-2 training raises or deforms the implicit basin supporting a Task-1 sample, then a noisy version of that sample is less likely to be mapped back to the original point by the post-Task-2 reverse dynamics.
This leads to a larger reconstruction error.
Conversely, samples that remain well supported by the post-Task-2 landscape should remain easier to reconstruct.
Thus, while $\mathcal{F}_{t^\star}(x)$ is not itself an energy difference, it probes the same local loss of support that intrinsic forgetting is intended to capture.


\section{Hyperparameters}
\label{app:hyperparams}

\paragraph{Synthetic MHN experiment.}
$N=12$ cluster memories, $d=50$, cosine $c=0.35$, inverse temperature $\beta=8$, rotation size $\varepsilon=0.02$, $n_{\text{rot}}=2000$ random antisymmetric $\Omega$, fixed-point tolerance $10^{-13}$ (\texttt{scripts/mhn\_equal\_angler\_replay.py}).

\paragraph{Synthetic MHN-BM experiment.}
$N=12$ cluster memories, $d=50$, $c=0.35$, $\beta=16$ (consistent in loss, update, and energy), data scale $2.0$, rotation $\varepsilon_{\mathrm{rot}}=0.02$, $n_{\mathrm{hidden}}=N+1$, Adam with learning rate $0.02$, batch size $64$.

\paragraph{Stable Diffusion split CIFAR-10.}
Backbone: \texttt{runwayml/stable-diffusion-v1-5}; only the UNet is fine-tuned, VAE frozen. Task~1 classes: bird, cat, deer, dog, horse. Task~2 classes: airplane, automobile, frog, ship, truck. 25{,}000 samples per task, 32$\times$32 images. Replay buffer $K=5000$; replay selection $\in\{\text{no replay, random, high\_energy}\}$. Forgetting metric: reconstruction MSE at $t=0.8$.

\paragraph{Pixel-space diffusion split CIFAR-10.}
PixelCNN-UNet trained from scratch in pixel space ($32\times 32$ RGB), ``standard'' channels $(128,256,512)$ ($\sim$45M params). Same living/non-living split as the Stable Diffusion setup, $25{,}000$ samples per task. AdamW (lr $2\times 10^{-4}$, weight decay $10^{-4}$, $500$ warmup steps), fp16, batch $512$, $2000$ epochs/task, EMA decay $0.9999$. Linear noise schedule ($T=1000$, $\beta\in[10^{-4},2\!\times\!10^{-2}]$), $50$ inference steps. Replay buffer $K=5000$; replay selection $\in\{\text{no replay, random, high\_energy, low\_energy}\}$. Forgetting metric: reconstruction MSE at $t=400$.

\paragraph{Pixel-space diffusion rotated MNIST.}
Same PixelCNN-UNet adapted to $28\times 28$ grayscale. Two tasks defined by \texttt{rotation\_angles}~$=[0^\circ,30^\circ]$. Optimizer, EMA, and noise schedule identical to the pixel-space CIFAR-10 setup; batch $256$ for $200$ epochs/task, no augmentation. Replay buffer $K=5000$; replay selection $\in\{\text{no replay, random, high\_energy, low\_energy}\}$. Forgetting metric: reconstruction MSE at $t=400$.

\end{document}